%% file: arxiv_v0.tex
\documentclass[11pt,letterpaper]{article}
\usepackage[margin=1in]{geometry}
\usepackage[T1]{fontenc}
\usepackage{newtxtext}
\usepackage{helvet}
\usepackage{courier}
\usepackage[hyphens]{url}
\usepackage{graphicx}
\usepackage{natbib}
\usepackage{caption}
\usepackage{algorithm}
\usepackage{algorithmic}

\usepackage{newfloat}
\usepackage{listings}
\DeclareCaptionStyle{ruled}{labelfont=normalfont,labelsep=colon,strut=off}
\floatstyle{ruled}
\newfloat{listing}{tb}{lst}{}
\floatname{listing}{Listing}

\usepackage{booktabs}

\usepackage{amsmath,amsfonts}
\usepackage{amsthm}
\usepackage[inline]{enumitem}

\renewcommand\cite{\citep}

\DeclareMathOperator*{\argmin}{arg\,min}
\DeclareMathOperator{\diag}{diag}
\DeclareMathOperator{\KL}{KL}
\DeclareMathOperator{\LSE}{LSE}
\DeclareMathOperator{\support}{supp}

\newtheorem{proposition}{Proposition}
\newtheorem{definition}{Definition}

\newtheorem{remark}{Remark}

\usepackage[table]{xcolor}
\definecolor{LightGray}{gray}{0.95}

\usepackage[hidelinks]{hyperref}

\title{SinkSLOT: Sinkhorn via Sparse Lifted Optimal Transport}
\author{
  Ian Hsieh$^{1,*}$, Soumya Snigdha Kundu$^{2}$, Tom Vercauteren$^{2,\dagger}$, Reuben Dorent$^{1,\dagger}$\\[6pt]
  \parbox{0.85\textwidth}{\centering\small
  $^{1}$Sorbonne Universit\'e, Institut du Cerveau - Paris Brain Institute - ICM, CNRS, Inria, Inserm, AP-HP, H\^opital de la Piti\'e Salp\^etri\`ere, 75013 Paris, France\\
  $^{2}$King's College London, UK\\[4pt]
  $^{*}$Corresponding author: \texttt{ian-yee-yang.hsieh@inria.fr}.\\
  $^{\dagger}$Equal contribution.}
}
\date{}

\begin{document}

\maketitle

\begin{abstract}
Entropic optimal transport (EOT) has been shown to offer a computationally tractable approximation to exact optimal transport. However, the standard Sinkhorn-Knopp algorithm has two main limitations. First, given discrete measures with $N$ points, each iteration requires $O(N^2)$ operations, which restricts its use on large-scale datasets (e.g. $N\geq10^4$). Second, it uses the independent coupling as a reference measure for regularisation. This assigns mass to high-cost transport edges at moderate regularisation strengths. We propose SinkSLOT, which addresses both limitations by putting forth the expected sliced lifted transport plan as a natural way to sparsify the Gibbs kernel with a non-independent prior coupling. We prove that: 1) SinkSLOT converges; 2) with $L$ slices, each resulting sparse Sinkhorn iteration costs $O(LN)$; and 3) the resulting objective is a divergence requiring no debiasing. Experiments on synthetic benchmarks show that SinkSLOT delivers substantial speedups over state-of-the-art dense and sparse EOT methods. We also demonstrate the applicability of the proposed divergence in a gradient flow experiment. The code is publicly available at \textcolor{blue}{\url{https://github.com/cai4cai/SinkSLOT}}.
\end{abstract}

\begin{figure*}[t]
\centering

\includegraphics[
width=1\textwidth,]{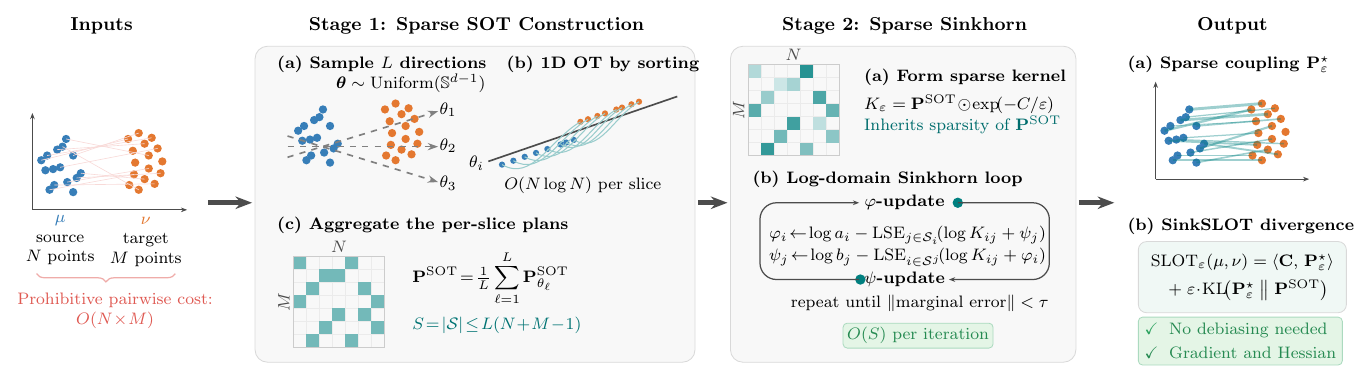}
  \caption{%
    \textbf{Overview of the SinkSLOT pipeline.}
    \textbf{Inputs:} two discrete measures $\mu$ and $\nu$;
    the na\"ive dense coupling has $N\times M$ nonzeros.
    \textbf{Stage~1:} $L$ random directions are sampled and
    1D OT problems are solved by sorting in $O(N\log N)$ per slice;
    the resulting plans are aggregated into a sparse lifted plan
    $\mathbf{P}^{\mathrm{SOT}}$ with at most $L(N+M-1)$ nonzeros.
    \textbf{Stage~2:} a sparse Gibbs kernel
    $K_\varepsilon = \mathbf{P}^{\mathrm{SOT}} \odot \exp(-C/\varepsilon)$
    inherits this sparsity, reducing each log-domain Sinkhorn
    iteration from $O(NM)$ to $O(L(N+M))$.
    \textbf{Output:} the sparse coupling $\mathbf{P}^\star_\varepsilon$
    and the SinkSLOT divergence $\mathrm{SLOT}_\varepsilon$,
    which requires no debiasing.
    }
\label{fig:overview}
\end{figure*}

\section{Introduction}
Optimal transport (OT) and the associated Wasserstein distance (WD) have become fundamental tools in machine learning and beyond~\cite{peyre2019computational,villani2009optimal}.
However, computing exact OT requires solving a linear programme with $O(N^3 \log N)$ complexity for discrete measures supported on $N$ points, limiting its applicability to large-scale problems.
This computational cost is particularly acute when WD is used as a loss function within an optimisation pipeline, such as training a generative neural network, where its gradient must be computed at every optimisation step.

Entropic OT (EOT)~\citep{cuturi2013sinkhorn} offers a computationally tractable approximation for the OT plan by adding entropy regularisation. The resulting strictly convex objective admits a unique regularised transport plan, computable via the Sinkhorn-Knopp algorithm~\citep{sinkhorn1967concerning}. For two $N$-point measures, standard dense Sinkhorn iterations require $O(N^2)$ time.
The resulting runtime is thus substantially lower than 
exact OT but still computationally demanding for large datasets.

Sliced optimal transport~(SOT)~\cite{bonneel2015sliced} provides a computationally attractive alternative to EOT, reducing runtime to $O(N \log N)$ per projection through 1D sorting.
However, it can underperform exact OT as a distance measure~\citep{rowland2019orthogonal}.
The lack of a transport plan in the original space further limits its utility when an explicit coupling is needed.
Recently, slice-based approaches that recover a transport plan by
lifting the 1D matchings back to the original space have been proposed~\cite{liu2025expected}.

Recently, \citet{nguyen2026sliced}  merges these two frameworks by replacing the reference coupling of EOT with the lifted SOT plan, introducing sliced-regularised optimal transport (SROT). SROT provides more accurate approximations of the true OT plan than EOT under the same level of regularisation, and significantly outperforms SOT. Both SROT and EOT are solved via Sinkhorn-type algorithms,
providing efficient approximations of the OT plan at reduced cost compared to exact solvers. 
However, the computational complexity of SROT remains $O(N^2)$ per Sinkhorn iteration due to the dense kernel,
substantially higher than the $O(N \log N)$ complexity of SOT.
The primary objective of \citet{nguyen2026sliced} is thus to improve coupling
accuracy over EOT, rather than to retain the computational efficiency of SOT.

In this work, we build upon \citet{nguyen2026sliced} and show that the sparsity of the SOT plan can be leveraged to formulate an efficient sparse Sinkhorn-type algorithm.
Rather than operating on the full $N \times N$ coupling matrix, the Sinkhorn iterations are restricted to the sparse support defined by the SOT plan, reducing per-iteration complexity from $O(N^2)$ to $O(LN)$, where $L$ is the number of projection slices.
As a result, the per-iteration cost remains substantially below EOT, making the method significantly more scalable than dense Sinkhorn while retaining the improved coupling accuracy of SROT.
Our contributions are threefold.
\begin{enumerate*}
    \item We formalize SinkSLOT, the sparse version of \citet{nguyen2026sliced}, and prove key properties, including convergence and uniqueness of the solution.
    \item We show that the SinkSLOT objective is a divergence without requiring EOT debiasing terms and admits an analytical gradient, making it suitable for optimisation and learning tasks.
    \item Experiments on synthetic benchmarks demonstrate that SinkSLOT achieves substantial speedups over state-of-the-art dense and sparse EOT methods.
\end{enumerate*}

\section{Background and related work}

In this section, we cover the necessary background on discrete OT for this paper. Readers are referred to \citet{peyre2019computational} for a comprehensive review, and to \citet{peyre2025optimal} for a treatment more geared toward machine learning. \citet{nguyen2025introduction} serves a similar purpose for the sliced problem.

\paragraph{Optimal Transport (OT)}
Given two discrete probability measures
$\mu = \sum_{i=1}^{N} a_i\,\delta_{x_i}$ and
$\nu = \sum_{j=1}^{M} b_j\,\delta_{y_j}$,
with positive unitary weight vectors $\mathbf{a}\in\Delta^N$, $\mathbf{b}\in\Delta^M$ 
and support points $x_i,y_j\in\mathbb{R}^d$,
and given a ground cost matrix $C\in\mathbb{R}_+^{N\times M}$ with entries
$C_{ij}=c(x_i,y_j)$, typically $C_{ij}=\|x_i-y_j\|^p$, $p\ge1$,
the OT problem~\citep{villani2009optimal,peyre2019computational} seeks a cost-minimising coupling:
\begin{equation}\label{eq:ot}
  P^{\mathrm{OT}}
  \;\in\;
  \argmin_{P\in\Pi(\mathbf{a},\mathbf{b})}\;
  \langle C,P\rangle,
\end{equation}
where
$\Pi(\mathbf{a},\mathbf{b})
=\bigl\{P\in\mathbb{R}_+^{N\times M}\mid
P\mathbf{1}_M=\mathbf{a},\;P^\top\mathbf{1}_N=\mathbf{b}\bigr\}$
is the set of admissible couplings (joint measures) with marginals $\mathbf{a}$ and $\mathbf{b}$.
When $C_{ij}=\|x_i-y_j\|^p$ with $p\ge1$, the optimal coupling $P^{\mathrm{OT}}$ defines the $p$-Wasserstein distance
$W_p(\mu,\nu)=\bigl(\langle C,P^{\mathrm{OT}}\rangle\bigr)^{1/p}$.
Problem \eqref{eq:ot} is a linear programme whose best known general-purpose solvers require $O(NM(N+M)\log(N+M))$ arithmetic operations~\citep{peyre2019computational},
limiting its use as a loss function within large-scale optimisation pipelines.

\paragraph{Entropic Optimal Transport (EOT)}
To alleviate this cost, \citet{cuturi2013sinkhorn} popularised
an entropy-regularised relaxation of~\eqref{eq:ot}:
\begin{equation}\label{eq:eot}
  P_{\varepsilon}^{\mathrm{EOT}}
  \;=\;
  \argmin_{P\in\Pi(\mathbf{a},\mathbf{b})}\;
  \langle C,P\rangle
  \;+\;
  \varepsilon\,\KL\bigl(P\,\|\,\mathbf{a}\otimes\mathbf{b}\bigr),
\end{equation}
where $\varepsilon>0$ is a regularisation strength,
$\mathbf{a}\otimes\mathbf{b}$ denotes the independent 
coupling
with entries $(\mathbf{a}\otimes\mathbf{b})_{ij}=a_i b_j$,
and
$\KL(P\|Q)=\sum_{ij}P_{ij}\log\frac{P_{ij}}{Q_{ij}}
 -P_{ij}+Q_{ij}$
is the generalised Kullback-Leibler divergence.
The unique solution has the form
$P_\varepsilon^{\mathrm{EOT}}=\diag(\mathbf{u})\,K\,\diag(\mathbf{v})$,
where $K_{ij}=a_i b_j\exp(-C_{ij}/\varepsilon)$ is the Gibbs kernel
and $(\mathbf{u},\mathbf{v})$ are positive scaling vectors found by
the Sinkhorn-Knopp algorithm:
\begin{equation}\label{eq:sinkhorn_eot}
  \mathbf{u}^{\ell+1}
  =\mathbf{a}\oslash(K\mathbf{v}^\ell),
  \qquad
  \mathbf{v}^{\ell+1}
  =\mathbf{b}\oslash(K^\top\mathbf{u}^{\ell+1}),
\end{equation}
where $\oslash$ is elementwise division.
The solution defines a discrepancy functional:
\begin{equation}\label{eq:eot_val}
  \mathrm{OT}_\varepsilon(\mu,\nu)
  =
  \langle C,P_\varepsilon^{\mathrm{EOT}}\rangle
  +\varepsilon\,\KL(P_\varepsilon^{\mathrm{EOT}}\|\mathbf{a}\otimes\mathbf{b}).
\end{equation}

Despite its widespread adoption, EOT has three well-known limitations.
First, although, for structured $C$ matrices, streaming backends help avoid the $O(NM)$ memory to store~$K$ at the cost of extra online computations \citep{feydy2019interpolating, ye2026flashsinkhorn}, each Sinkhorn iteration requires two dense matrix-vector products, resulting in $O(NM)$ cost per iteration. This quadratic scaling is prohibitive for
very large point clouds (e.g.\ $N,M\ge10^4$). 
Second, regularising toward the cost-agnostic independent coupling
$\mathbf{a}\otimes\mathbf{b}$ biases the plan toward dense and diffuse
matchings that potentially assign non-negligible mass to higher-cost transport
edges,
particularly for moderate values of~$\varepsilon$.
Finally, the EOT cost is biased in the sense that $\mathrm{OT}_\varepsilon(\mu,\mu)\neq0$
in general, so that using it directly as a loss requires an additional debiasing step to obtain the Sinkhorn divergence~\citep{feydy2019interpolating}.

\paragraph{Sparsification techniques for OT}Several methods reduce the complexity of OT by imposing sparsity or low-dimensional structure on the transport plan. We refer to~\citet{khamis2024scalable,ouyang2026sparsification} for broad surveys. \citet{blondel2018smooth} induce sparse transport plans through squared-$L_2$ and group-lasso regularisation, while \citet{liu2023sparsity} constrain each column to contain at most $k$ nonzeros. Low-rank alternatives include factored couplings \citep{forrow2019statistical,scetbon2021low,scetbon2022low} and low-rank-plus-sparse decompositions \citep{liu2021approximating}. These methods require specialised optimisation procedures rather than standard Sinkhorn scaling on a fixed kernel.

More closely related to our work, sparse-kernel methods reduce the $O(NM)$ cost of Sinkhorn iterations. \citet{schmitzer2019stabilized} adaptively truncate the kernel using the current dual potentials, \citet{gasteiger2021scalable} retain local interactions via locality-sensitive hashing, and \citet{li2023importance} pre-build a sparse kernel by importance sampling so each Sinkhorn iteration costs $O(|\mathcal S|)$ for support $\mathcal S$. However, the latter requires
$O(NM)$ preprocessing to construct their sparse kernel. 
Furthermore, these methods all use the independent coupling as the reference measure, so the value of the nonzero entries do not contain geometric information. 

\paragraph{Reference coupling beyond $\mathbf{a}\otimes\mathbf{b}$}

Non-product reference measures have been considered in the Schrödinger-bridge literature~\citep{leonard2014survey}, though they have received little attention in static EOT. \citet{freulon2025entropic} study EOT with a Gaussian reference coupling, but restrict their analysis to Gaussian marginals.

Another framework to efficiently approximate OT is sliced optimal transport (SOT)~\citep{bonneel2015sliced},  which averages 1D Wasserstein distances over random projections, each solved via sorting, giving $O(L(N+M)\log(N+M))$ cost for $L$ directions.
To produce a transport plan in the original space, \citet{muzellec2019subspace,liu2025expected,tanguy2025sliced}
propose \emph{lifting} the 1D optimal matchings back to the original
support points.
For each direction $\theta\in\mathbb{S}^{d-1}$, let
$P_\theta^{\mathrm{SOT}}\in\Pi(\mathbf{a},\mathbf{b})$ denote the
coupling that solves OT in the projected 1D
space.
Averaging over $L$ directions yields the
expected sliced transport plan:
$P^{\mathrm{SOT}}
  =\frac{1}{L}\sum_{\ell=1}^{L}P_{\theta_\ell}^{\mathrm{SOT}}$.
By construction,
$P^{\mathrm{SOT}}\in\Pi(\mathbf{a},\mathbf{b})$ is sparse.
Each 1D plan contributes at most $N+M-1$ nonzero entries~\citep[Theorem~8.1.2]{brualdi2006combinatorial},
so $|\mathcal{S}|=O\bigl(L(N+M)\bigr)$ where $\mathcal{S}=\support(P^{\mathrm{SOT}})$,
which is much smaller than $NM$ when $L\ll\min(N,M)$.
\citet{nguyen2026sliced} observed that the sliced plan
$P^{\mathrm{SOT}}$ is a more informative prior than the independent
coupling and proposed sliced-regularised OT (SROT):
\begin{equation}\label{eq:srot}
  P_{\varepsilon,\gamma}^{\mathrm{SROT}}
  =\argmin_{P\in\Pi(\mathbf{a},\mathbf{b})}
  \langle C,P\rangle
  +\varepsilon\,\KL\bigl(P\,\|\,P^{\mathrm{SOT}}_\gamma\bigr),
\end{equation}
where $P^{\mathrm{SOT}}_\gamma
=(1-\gamma)P^{\mathrm{SOT}}+\gamma\,\mathbf{a}\otimes\mathbf{b}$
is smoothed with $\gamma>0$ to guarantee full support.
SROT yields couplings closer to the true OT plan
than EOT under the same regularisation level,
but the smoothing destroys the sparsity of $P^{\mathrm{SOT}}$.
Each Sinkhorn iteration still costs $O(NM)$, the same as EOT.
Similar to \eqref{eq:eot_val}, the optimal solution leads to a biased discrepancy functional $\mathrm{SROT}_{\varepsilon,\gamma}(\mu,\nu)$.

\section{SinkSLOT: Sinkhorn via Sparse Lifted OT}
\label{sec:sinkslot}

\paragraph{Problem formulation}
\label{sec:formulation}

We propose to solve the SROT problem~\eqref{eq:srot} with the
\emph{unsmoothed}, i.e.\ $\gamma=0$, expected slice transport reference plan $P^{\mathrm{SOT}}$:
\begin{equation}\label{eq:sinkslot}
  P_\varepsilon^{\star}
  =
  \argmin_{P\in\Pi(\mathbf{a},\mathbf{b})}\;
  \langle C,P\rangle
  +\varepsilon\,\KL\bigl(P\,\|\,P^{\mathrm{SOT}}\bigr),
\end{equation}
Absorbing the transport cost into the KL divergence, we obtain the
equivalent formulation
\begin{equation}\label{eq:slot_kl}
  P_\varepsilon^{\star}
  =
  \argmin_{P\in\Pi(\mathbf{a},\mathbf{b})}\;
  \KL\bigl(P\,\|\,K_\varepsilon\bigr),
\end{equation}
where $K_\varepsilon=P^{\mathrm{SOT}}\odot\exp(-C/\varepsilon)$
is the \emph{sparse} Gibbs kernel.
The SinkSLOT discrepancy associated with a SOT plan is
\begin{equation}\label{eq:slot_val}
  \mathrm{SLOT}_\varepsilon(\mu,\nu)
  \;=\;
  \langle C,P_\varepsilon^{\star}\rangle
  +\varepsilon\,\KL(P_\varepsilon^{\star}\|P^{\mathrm{SOT}}).
\end{equation}

Setting $\gamma=0$ confers key advantages over the smoothed
formulation.
Assuming no overlapping projected samples (zero probability under sufficient numerical precision), each 1D optimal matching contributes at most $N+M-1$ nonzero entries to its corresponding $P_\theta^{\mathrm{SOT}}$. 
Consequently, the support of the aggregated lifted plan, $\mathcal{S}:=\support(P^{\mathrm{SOT}})$, satisfies
$|\mathcal{S}|\leq L(N+M-1)$.
Since the Gibbs kernel $K_\varepsilon$ inherits the sparsity pattern of $P^{\mathrm{SOT}}$, the cost of each Sinkhorn iteration is reduced from $O(NM)$ to $O(|\mathcal{S}|)$, and to $O\bigl(L(N+M)\bigr)$ in the worst case.
To avoid infinite costs, optimising the KL divergence forces
$\support(P)\subseteq\mathcal{S}$
for any feasible coupling $P$ with finite objective value. Hence, the minimisation in~\eqref{eq:slot_kl} is implicitly restricted to couplings supported on~$\mathcal{S}$.

\paragraph{Existence, uniqueness, and convergence of sparse Sinkhorn}
\label{sec:convergence}

We first establish that~\eqref{eq:sinkslot} admits a unique solution
of diagonal-scaling form, then show that the sparse Sinkhorn
algorithm converges to it.

\begin{proposition}[Unique diagonal scaling]\label{prop:unique}
Let $K_\varepsilon=P^{\mathrm{SOT}}\odot\exp(-C/\varepsilon)$ be the
sparse Gibbs kernel with support
$\mathcal{S}=\support(P^{\mathrm{SOT}})$.
The SinkSLOT problem~\eqref{eq:slot_kl} has a unique solution of the form
\begin{equation}\label{eq:scaling}
  P_\varepsilon^{\star}
  =\diag(\mathbf{u})\,K_\varepsilon\,\diag(\mathbf{v}),
\end{equation}
for positive scaling vectors
$\mathbf{u}\in\mathbb{R}_{>0}^N$ and
$\mathbf{v}\in\mathbb{R}_{>0}^M$.
\end{proposition}
\noindent The proof is in Appendix~\ref{app:proof_unique}.

\begin{proposition}[Convergence of sparse Sinkhorn]\label{prop:sinkhorn}
Let $(\mathbf{u}^\ell,\mathbf{v}^\ell)$ be defined by
\begin{equation}\label{eq:sparse_sinkhorn}
  \mathbf{u}^{\ell+1}
  =\mathbf{a}\oslash(K_\varepsilon\,\mathbf{v}^\ell),
  \qquad
  \mathbf{v}^{\ell+1}
  =\mathbf{b}\oslash(K_\varepsilon^\top\mathbf{u}^{\ell+1}),
\end{equation}
initialised at $\mathbf{u}^0=\mathbf{v}^0=\mathbf{1}$.
$P^\ell=\diag(\mathbf{u}^\ell)\,K_\varepsilon\,\diag(\mathbf{v}^\ell)$
converges to the unique solution~$P_\varepsilon^{\star}$
of~\eqref{eq:sinkslot}.
\end{proposition}
\noindent The proof is in Appendix~\ref{app:proof_sinkhorn}.

\paragraph{SinkSLOT pipeline, computational complexity}
\label{sec:complexity}

Figure~\ref{fig:overview} and Algorithm~\ref{alg:sinkslot} summarise SinkSLOT.
The total cost decomposes into (i) construction of the sparse SOT plan
and (ii) sparse Sinkhorn iterations.

\begin{algorithm}[t]
\caption{SinkSLOT pseudo-code}
\label{alg:sinkslot}
\begin{algorithmic}[1]
\REQUIRE
    Source $\{(x_i,a_i)\}_{i=1}^N$,
    target $\{(y_j,b_j)\}_{j=1}^M$,
    $L$ projections,
    regularisation $\varepsilon>0$,
    and tolerance $\tau>0$
\ENSURE
    Sparse coupling $P_\varepsilon^\star$ and discrepancy
    $\mathrm{SLOT}_\varepsilon(\mu,\nu)$

\STATE Sample
    $\theta_1,\dots,\theta_L
    \overset{\mathrm{i.i.d.}}{\sim}
    \mathcal{U}(\mathbb{S}^{d-1})$

\FOR{$\ell=1,\dots,L$}
    \STATE Project:
        $\hat{x}_i=\theta_\ell^\top x_i$ and
        $\hat{y}_j=\theta_\ell^\top y_j$
    \STATE Solve the 1D OT problem by sorting
        $\to P_{\theta_\ell}^{\mathrm{SOT}}$
\ENDFOR

\STATE
    $P^{\mathrm{SOT}}
    =\tfrac{1}{L}\sum_{\ell=1}^L
    P_{\theta_\ell}^{\mathrm{SOT}}$
    \hfill
    \COMMENT{sparse, $\mathrm{nnz}\leq L(N+M)$}

\STATE Compute cost $C_{ij} = c(x_i,y_j)$,
    $\forall (i,j)\in\support(P^{\mathrm{SOT}})$

\STATE Form sparse log-kernel:
    $\log(K_\varepsilon)_{ij}
    =\log P^{\mathrm{SOT}}_{ij}-C_{ij}/\varepsilon$

\STATE Initialise
    $\boldsymbol{\varphi}^{(0)}=\mathbf{0}_N$ and
    $\boldsymbol{\psi}^{(0)}=\mathbf{0}_M$

\REPEAT
    \STATE
        $\varphi_i^{(\ell+1)}
        =\log a_i
        -\LSE_{j\in\mathcal{S}_i}
        \big(
            \log(K_\varepsilon)_{ij}+\psi_j^{(\ell)}
       \big ),
        \forall i$

    \STATE
        $\psi_j^{(\ell+1)}
        =\log b_j
        -\LSE_{i\in\mathcal{S}^j}
        \big(
            \log(K_\varepsilon)_{ij}+\varphi_i^{(\ell+1)}
        \big),
        \forall j$
\UNTIL{marginal violation $<\tau$}

\STATE
    $P_\varepsilon^\star
    =\operatorname{diag}(e^{\boldsymbol{\varphi}})
    K_\varepsilon
    \operatorname{diag}(e^{\boldsymbol{\psi}})$

\STATE Compute
    $\mathrm{SLOT}_\varepsilon(\mu,\nu)$
    from $P_\varepsilon^\star$

\RETURN
    $P_\varepsilon^\star$ and
    $\mathrm{SLOT}_\varepsilon(\mu,\nu)$
\end{algorithmic}
\end{algorithm}

\begin{table}[t]
\centering
\caption{Computational complexity of the reference coupling and each Sinkhorn iteration.
Here, $S=|\mathcal{S}|\leq L(N+M-1)$ denotes the number of nonzeros in the SOT plan,
and we assume $N=M$ for simplicity.}
\label{tab:complexity}
\begin{tabular}{@{}lccc@{}}
\toprule
\textbf{Method} &
\textbf{Reference coupling} &
\textbf{Reference coupling cost} &
\textbf{Sinkhorn iter.} \\
\midrule
EOT~\citep{cuturi2013sinkhorn}
& $\mathbf{a}\otimes\mathbf{b}$
& $O(N^2)$
& $O(N^2)$ \\

SROT~\citep{nguyen2026sliced}
& $(1-\gamma)P^{\mathrm{SOT}}+\gamma\,\mathbf{a}\otimes\mathbf{b}$
& $O(LNd+LN\log N+N^2)$
& $O(N^2)$ \\

\textbf{SinkSLOT (ours)}
& $P^{\mathrm{SOT}}$
& $O(LNd+LN\log N)$
& $O(S)$ \\
\bottomrule
\end{tabular}
\end{table}

Stage 1 corresponds to the SOT construction.
For each of the $L$ directions, projecting $N+M$ points costs
$O\bigl((N+M)d\bigr)$, while sorting the projected source and target points costs
$O(N\log N+M\log M)$. Constructing the resulting 1D transport plan requires $O(N+M)$ operations and produces at most $N+M-1$ nonzero entries. Accumulating the $L$ sparse plans therefore costs
$O\bigl(L(N+M)\bigr)$. Hence, the total cost of this stage is
$O\!\left(
L\bigl((N+M)d+N\log N+M\log M\bigr)
\right)$.

In Stage 2, the sparse Sinkhorn is applied.
Each iteration of~\eqref{eq:sparse_sinkhorn} performs two sparse
matrix-vector products on $K_\varepsilon$,
which has $|\mathcal{S}|\leq L(N+M)$ nonzero entries, costing $O(|\mathcal{S}|)$. 
For $T$ iterations until convergence,
the Sinkhorn stage costs $O\bigl(T\cdot L(N+M)\bigr)$. 
The cost entries $C_{ij}$ are computed \emph{only} for $(i,j)\in\mathcal{S}$ at $O(d)$ per entry, adding $O(|\mathcal{S}|\,d)$.

For small~$\varepsilon$, the entries of $K_\varepsilon$ can span
many orders of magnitude.
We work entirely in the log domain,
defining 
$\boldsymbol{\varphi}^\ell=\log\mathbf{u}^\ell$,
$\boldsymbol{\psi}^\ell=\log\mathbf{v}^\ell$
and $\log(K_\varepsilon)_{ij} = \log P^{\mathrm{SOT}}_{ij}
      -\tfrac{C_{ij}}{\varepsilon}$.
The Sinkhorn updates then read
\begin{align}
  \varphi_i^{(\ell+1)}
  &=\log a_i
   -\mathrm{LSE}_{j\in\mathcal{S}_i}
    \Bigl(\log(K_\varepsilon)_{ij}
      +\psi_j^{(\ell)}\Bigr),
  \label{eq:logrow}\\
  \psi_j^{(\ell+1)}
  &=\log b_j
   -\mathrm{LSE}_{i\in\mathcal{S}^j}
    \Bigl(\log(K_\varepsilon)_{ij}
      +\varphi_i^{(\ell+1)}\Bigr),
  \label{eq:logcol}
\end{align}
where $\mathcal{S}_i=\{j:(i,j)\in\mathcal{S}\}$,
$\mathcal{S}^j=\{i:(i,j)\in\mathcal{S}\}$, log-sum-exp
($\mathrm{LSE}$) uses an online max-shift trick for stability~\cite{ye2026flashsinkhorn,milakov2018online,nowozin2016streaming}.

The total complexity is
$O\bigl(
L\bigl((N+M)(d+T)+N\log N+M\log M\bigr)
\bigr)$.
In terms of memory, 
only the sparse kernel ($O(|\mathcal{S}|)$),
the scaling vectors ($O(N+M)$), and the cost entries on
$\mathcal{S}$ ($O(|\mathcal{S}|)$) need to be stored. This sparsity is intrinsic to our formulation. Even without matrix-free backends such as KeOps \citep{feydy2019interpolating} or fused Triton kernels \citep{ye2026flashsinkhorn}, the full $N\times M$ cost or kernel matrices are never formed.

\paragraph{SinkSLOT as a divergence}
\label{sec:divergence}

A practical advantage of the unsmoothed formulation is that the SinkSLOT
functional~\eqref{eq:slot_val} is already a well-behaved discrepancy
\emph{without debiasing}.
Recall that with standard entropic regularisation,
$\mathrm{OT}_\varepsilon(\mu,\mu)>0$ in general.
\citet{feydy2019interpolating} proposed to subtract self-transport terms to obtain a proper divergence:
\begin{equation*}
S_{\varepsilon}(\alpha,\beta)
=
\operatorname{OT}_{\varepsilon}(\alpha,\beta)
-\frac{1}{2}\operatorname{OT}_{\varepsilon}(\alpha,\alpha)
-\frac{1}{2}\operatorname{OT}_{\varepsilon}(\beta,\beta)
\end{equation*}
The smoothed $\mathrm{SROT}_{\varepsilon,\gamma}$ similarly requires
debiasing. 
In contrast, we show that SinkSLOT bypasses these issues entirely.

\begin{proposition}[$\mathrm{SLOT}_\varepsilon$ is a divergence]\label{prop:divergence}
Assume $C_{ij}=\|x_i-y_j\|^p$ for some $p\ge1$,
and that the support points of $\mu$ and $\nu$ are in
general position.\footnote{%
  General position ensures that, for almost every projection
  direction, projected coordinates are distinct, so the 1D OT
  matching is uniquely determined.}
Then, for a fixed set of projection directions,
$\mathrm{SLOT}_\varepsilon$ satisfies:
\begin{enumerate}
\item \textbf{Non-negativity.}\;
  $\mathrm{SLOT}_\varepsilon(\mu,\nu)\ge0$.
\item \textbf{Identity of indiscernibles.}\;
  $\mathrm{SLOT}_\varepsilon(\mu,\nu)=0
   \Leftrightarrow \mu=\nu$.
\item \textbf{Symmetry.}\;
  $\mathrm{SLOT}_\varepsilon(\mu,\nu)
   =\mathrm{SLOT}_\varepsilon(\nu,\mu)$.
\end{enumerate}
In particular, $\mathrm{SLOT}_\varepsilon$ is a divergence
in the sense of \citet{feydy2019interpolating}, with no
debiasing required.
\end{proposition}
\noindent The proof is in Appendix~\ref{app:proof_divergence}.

\paragraph{Linear convergence rate}
\label{sec:rate}

Having established convergence of the sparse Sinkhorn iterates
(Proposition~\ref{prop:sinkhorn}), we now 
show that convergence is linear (geometric) and give an explicit expression for the
asymptotic contraction factor.

\begin{definition}[Support graph]\label{def:support_graph}
The \emph{bipartite support graph}
$G_\mathcal{S}=(U\cup V,\,\mathcal{S})$
has row vertices $U=\{1,\dots,N\}$, column vertices $V=\{1,\dots,M\}$,
and edge set $\mathcal{S}=\support(P^{\mathrm{SOT}})$.
\end{definition}

\begin{proposition}[Spectral gap]\label{prop:spectral_gap}
Define the normalised plan matrix
\begin{equation}\label{eq:normalised_plan}
  A = \diag(\mathbf{a})^{-1/2}\; P_\varepsilon^{\star}\;
  \diag(\mathbf{b})^{-1/2}
  \;\in\;\mathbb{R}^{N\times M},
\end{equation}
and let $\sigma_1\geq\sigma_2\geq\cdots\geq 0$ denote its singular
values.  Then:
\begin{enumerate}
\item $\sigma_1=1$, with left and right singular vectors
  $\sqrt{\mathbf{a}}$ and $\sqrt{\mathbf{b}}$, respectively.
\item $\sigma_2<1$ if and only if $G_\mathcal{S}$ is connected.
\end{enumerate}
\end{proposition}

\begin{proposition}[Linear convergence rate]\label{prop:rate}
When $G_\mathcal{S}$ is connected, the sparse Sinkhorn iterates
$P^\ell=\diag(\mathbf{u}^\ell)\,K_\varepsilon\,\diag(\mathbf{v}^\ell)$
converge linearly:
\begin{equation}\label{eq:rate}
  \limsup_{\ell\to\infty}\;
  \bigl\|P^\ell - P_\varepsilon^{\star}\bigr\|_1^{\,1/\ell}
  \;\leq\; \sigma_2^{\,2},
\end{equation}
where $\sigma_2$ is the second singular value of~$A$
from Proposition~\ref{prop:spectral_gap}.
In particular, the number of iterations to achieve
$\|P^\ell-P_\varepsilon^\star\|_1\leq\delta$ is at most
\begin{equation}\label{eq:iter_count}
  T \;=\; O\!\left(
    \frac{\log(1/\delta)}{\log(1/\sigma_2^2)}
  \right),
\end{equation}
giving a total arithmetic cost of
$O\!\bigl(L(N{+}M)\cdot\log(1/\delta)\,/\,\log(1/\sigma_2^2)\bigr)$.

If $G_\mathcal{S}$ is disconnected, the scaling problem decomposes into
independent subproblems on the connected
components~\citep{kalantari2008complexity}, and the above applies to
each with its own second singular value.
\end{proposition}
\noindent The proofs are in Appendix~\ref{app:proof_rate}.

The convergence rate $\sigma_2^2$ depends on $\varepsilon$ through the
optimal plan $P_\varepsilon^\star$.
As $\varepsilon\to0$, $P_\varepsilon^\star$ converges to the
KL-selected minimiser of the support-restricted OT problem
(Proposition~\ref{prop:limits} below), and $\sigma_2\to1$: convergence degenerates, a well-known difficulty of
solving near-linear programmes with Sinkhorn.
As $\varepsilon\to\infty$, $P_\varepsilon^\star\to P^{\mathrm{SOT}}$
and $\sigma_2$ is bounded away from~$1$ (for
connected~$G_\mathcal{S}$), giving fast convergence.
This mirrors the standard trade-off in entropic OT between
approximation quality (small~$\varepsilon$) and computational
efficiency (large~$\varepsilon$).

\paragraph{Limiting behaviour as
  $\varepsilon\to0$ and $\varepsilon\to\infty$}
\label{sec:limits}

The regularisation strength~$\varepsilon$ interpolates between two
natural extremes: the sliced plan $P^{\mathrm{SOT}}$
(large~$\varepsilon$) and a \emph{support-restricted} exact OT
solution (small~$\varepsilon$).

\begin{proposition}[Limiting behaviour]\label{prop:limits}
Let $\Pi_\mathcal{S}
=\{P\in\Pi(\mathbf{a},\mathbf{b})
  :\support(P)\subseteq\mathcal{S}\}$
denote the set of couplings supported on~$\mathcal{S}$.
\begin{enumerate}
\item \textbf{Small-$\varepsilon$ limit.}\;
  As $\varepsilon\to0^+$, $P_\varepsilon^\star$ converges to the unique coupling
  $\widehat{P}$ that minimises
  $\KL\bigl(P\,\|\,P^{\mathrm{SOT}}\bigr)$ among all solutions of
  the support-restricted OT problem.

\item \textbf{Large-$\varepsilon$ limit.}\;
  As $\varepsilon\to+\infty$,
  \begin{align}\label{eq:epsinf_limit}
    &P_\varepsilon^\star\;\to\;P^{\mathrm{SOT}}.
  \end{align}
\end{enumerate}
\end{proposition}
\noindent The proof is in Appendix~\ref{app:proof_limits}.
Proposition~\ref{prop:limits} reveals that SinkSLOT interpolates
continuously between the two interpretable endpoints.
In practice, a moderate~$\varepsilon$ offers the best trade-off:
a coupling close to true OT with rapid
convergence.

\section{Gradient and Hessian of the functional}
\label{sec:gradients}

When the $\mathrm{SLOT}_\varepsilon$ divergence is used as a loss function for
optimising the source positions~$X$, we require its gradient
$\nabla_X\mathrm{SLOT}_\varepsilon$ and, for second-order methods,
Hessian-vector products
$\mathcal{T}\,A=(\nabla^2_X\mathrm{SLOT}_\varepsilon)\,A$.
We derive both under the practical choice of treating the
SOT reference plan as a \emph{fixed} (detached) prior, and
discuss when this is justified.

\paragraph{Should we differentiate through $P^{\mathrm{SOT}}$?}
\label{sec:gradient_discussion}

The divergence $\mathrm{SLOT}_\varepsilon(\mu,\nu)$
depends on the source positions~$X$ through \emph{two} channels:
(i)~the cost matrix $C_{ij}=\|x_i-y_j\|^2$ and
(ii)~the SOT reference plan $P^{\mathrm{SOT}}(X,Y)$, which itself
depends on~$X$ through the 1D projections and rank-based matchings.
By the envelope theorem (Danskin's theorem), at any point~$X$ where
the projected rank orders are locally constant
(which holds for Lebesgue-almost every~$X$),
the gradient decomposes as
\begin{equation}\label{eq:full_grad}
  \frac{d\,\mathrm{SLOT}_\varepsilon}{dx_k}
  =
  \underbrace{\sum_j P^{\star}_{kj}\,\nabla_{x_k}
      C_{kj}}_{\text{(I) direct / Danskin}}
  +\;\underbrace{\varepsilon\sum_{(i,j)\in\mathcal{S}}
    \!\left(1 - \frac{P^\star_{ij}}{P^{\mathrm{SOT}}_{ij}}\right)
    \frac{\partial P^{\mathrm{SOT}}_{ij}}{\partial
    x_k}}_{\text{(II) reference-plan term}},
\end{equation}
where $P^\star$ is the SinkSLOT solution, dropping the explicit reference to $\varepsilon$ for notational conciseness.
In the classical EOT
setting~\citep{cuturi2013sinkhorn,feydy2019interpolating},
the reference coupling is the \emph{independent} product
$\mathbf{a}\otimes\mathbf{b}$, which does not depend on~$X$,
so term~(II) vanishes identically and the gradient reduces to
term~(I) alone.
In SinkSLOT, however, $P^{\mathrm{SOT}}$ depends on $X$ through
one-dimensional sorting operations, making term~(II) nonzero in
principle.
We advocate treating $P^{\mathrm{SOT}}$ as a \emph{fixed prior}
(i.e., stop-gradient / detach) and using only term~(I),
for the following reasons.

Term (II) is small for moderate~$\varepsilon$.
By Proposition~\ref{prop:limits},
$P^\star\to P^{\mathrm{SOT}}$ as $\varepsilon\to\infty$, so
$P^\star_{ij}/P^{\mathrm{SOT}}_{ij}\to 1$ and term~(II) vanishes.
At moderate~$\varepsilon$, the ratio stays close to~$1$,
especially on the high-mass entries of~$P^{\mathrm{SOT}}$ that
dominate the sum.
Conversely, as $\varepsilon\to0$, $P^\star$ approaches the
support-restricted OT (Proposition~\ref{prop:limits}), whose
gradient depends on the support set~$\mathcal{S}$ (a
combinatorial object) rather than the values of
$P^{\mathrm{SOT}}$.

$P^\mathrm{SOT}$ is locally constant.
Each 1D OT plan $P_\theta^{\mathrm{SOT}}$ is determined by the
rank order of the projected points $\theta^\top x_i$;
this rank order is a piecewise-constant function of~$X$,
introducing discontinuities in
$\partial P^{\mathrm{SOT}}/\partial X$.
Although relaxations exist, they add complexity and lose the
sparsity advantage.
Note that at differentiable points (Lebesgue-almost everywhere),
$\partial P^{\mathrm{SOT}}/\partial X = 0$,
so term~(II) vanishes \emph{exactly}, making stop-gradient not
merely practical but formally correct.

This choice is consistent with EOT practice.
The standard approach in EOT-based
learning~\citep{feydy2019interpolating,cuturi2013sinkhorn}
is to apply Danskin's theorem and differentiate only
through~$C$.
Under the stop-gradient convention, the gradient and Hessian of $\mathrm{SLOT}_\varepsilon$
have \emph{exactly} the same analytic form as those of
EOT~\citep{feydy2019interpolating,ye2026flashsinkhorn},
with the dense plan $P^{\mathrm{EOT}}$ replaced everywhere by the
sparse plan~$P^\star$.
All computational gains therefore come from the sparsity of~$P^\star$:
every transport-vector product costs $O(|\mathcal{S}|)$ instead of
$O(NM)$.

\paragraph{First-order gradient}
\label{sec:first_order}

\begin{proposition}[Gradient of SinkSLOT]\label{prop:gradient}
Assume the squared Euclidean cost $C_{kj}=\|x_k-y_j\|^2$ and
treat $P^{\mathrm{SOT}}$ as fixed.
Then the gradient of $\mathrm{SLOT}_\varepsilon$ with respect to the
source positions is
\begin{equation}\label{eq:grad_slot}
  \nabla_{x_k}\mathrm{SLOT}_\varepsilon
  = 2\!\!\sum_{j\in\mathcal{S}_k}\!P^\star_{kj}\,(x_k-y_j)
  = 2\,a_k\bigl(x_k-T_\varepsilon(x_k)\bigr),
\end{equation}
where  
  $T_\varepsilon(x_k)
  =\frac{1}{a_k}\sum_{j\in\mathcal{S}_k}P^\star_{kj}\,y_j$
is the barycentric projection and $\mathcal{S}_k=\{j:(k,j)\in\mathcal{S}\}$. In matrix form,
\begin{align}\label{eq:grad_matrix}
  \nabla_X\mathrm{SLOT}_\varepsilon
  &= 2\,\diag(\mathbf{a})\bigl(X-T_\varepsilon(X)\bigr),\\
  T_\varepsilon(X)
  &= \diag(\mathbf{a})^{-1}\,P^\star Y.
\end{align}
\end{proposition}
\noindent The proof is in Appendix~\ref{app:proof_gradient}.

\paragraph{Computational cost}
The sparse matrix-vector product $P^\star Y$ costs
$O(|\mathcal{S}|\,d)=O(L(N{+}M)\,d)$,
compared with $O(NMd)$ for the dense EOT gradient.
No sorting or projection is required at differentiation time.

\paragraph{Second-order Hessian and Hessian-vector products}
\label{sec:hessian}

Following the derivation of~\citet[Appendix~C]{ye2026flashsinkhorn}
and~\citet{li2025robust}, we derive the Hessian
$\nabla^2_X\mathrm{SLOT}_\varepsilon$ under stop-gradient in
Appendix~\ref{app:proof_hessian}. Since the analytic structure is
identical to EOT, the Hessian-vector product
$(\nabla^2_X\mathrm{SLOT}_\varepsilon)\,A$ can be computed via
sparse linear solves involving the sensitivity matrix
$H^\star\in\mathbb{R}^{(N+M)\times(N+M)}$, which inherits the
sparsity of~$P^\star$. Each Hessian-vector product therefore costs
$O(|\mathcal{S}|)$ per conjugate-gradient iteration, compared
with $O(NM)$ for dense EOT.

\section{Experiments}
\label{sec:experiments}
To assess the effectiveness of the proposed method, we evaluate
SinkSLOT both as a \emph{transport solver} and as a \emph{loss
function} via its divergence. We benchmark against
FlashSinkhorn~\citep{ye2026flashsinkhorn}, the most computationally
efficient GPU implementation of dense Sinkhorn available;
SROT~\citep{nguyen2026sliced}, which uses a smoothed sliced reference
coupling but operates on the full dense kernel; and a sparse Sinkhorn
approach, Spar-Sink~\citep{li2023importance}, which sparsifies the
Gibbs kernel by selecting entries via importance sampling.
Experiments are run on an NVIDIA H100-80GB. Implementation details and full results are  in
Appendix~\ref{app:implementation_details} and \ref{app:additional_results}.

\subsection{Synthetic Benchmarks}

\paragraph{Setup}
We consider five datasets with different geometric structures and dimensionalities, following previous experiments~\citep{nguyen2026sliced,ye2026flashsinkhorn}: (i)~\textbf{half-moons} ($d{=}2$), a pair of
interleaved crescent-shaped point clouds;
(ii)~\textbf{8-Gaussians} ($d{=}2$), a mixture of eight isotropic Gaussians; (iii)~\textbf{two-rings} ($d{=}2$), where $\mu$ and $\nu$ are supported on concentric circles with different radii 
(iv)~\textbf{Gaussian} $d{=}3$; and (v)~\textbf{Gaussian}
$d{=}64$. We sample $N=M=10^4$ points with weights drawn uniformly from $[0.1, 1.1]$, normalised to sum to one.
All methods use identical early stopping based on marginal violation (details in Appendix~\ref{app:benchmark_exp}). 
Accuracy is measured by the relative
transport cost gap (\%),
$\Delta:=100\times(\langle C, P \rangle - \langle C, P^{\mathrm{OT}} \rangle)\,/\,
\langle C, P^{\mathrm{OT}} \rangle$, where
$P^{\mathrm{OT}}$ is the exact OT plan from linear programming. 

\paragraph{SinkSLOT matches dense Sinkhorn accuracy at lower computational cost}

First, we assess whether sparsifying the Gibbs kernel with sliced lifted plans degrades plan quality.
We sweep over
regularisation strength $\varepsilon$ (all methods), slice count
$L$ (SROT and SinkSLOT), and subsampling budget $S_0$ (Spar-Sink). For each method, we record the minimum runtime required to reach a relative cost gap below three
thresholds: $\Delta \leq 1\%$, $5\%$, and $10\%$.
Table~\ref{tab:accuracy_speed} reports SinkSLOT's speedup over each competitor across all five datasets,
along with SinkSLOT's total runtime.

\input{Tables/speedup}

Three conclusions emerge. First, Spar-Sink fails to reach any gap
threshold on the Gaussian datasets ($d{=}3$; $d{=}64$) and cannot
attain the ${\leq}\,5\%$ gap on 8-Gaussians. Where Spar-Sink reaches the accuracy threshold, SinkSLOT is $2$--$45\times$ faster. Second, SROT is the slowest method overall: SinkSLOT achieves $15$--$160\times$ speedups across all datasets and
gap levels.
Third, and most importantly, SinkSLOT consistently outperforms
FlashSinkhorn by $3$ to $36\times$, with the largest gains on the low-dimensional benchmarks. Even on the
challenging $d{=}64$ Gaussian setting, SinkSLOT is
$3$--$11\times$ faster. These results show that SinkSLOT can reach the same accuracy as dense methods at a fraction of the computation time. While SinkSLOT primarly targets a reduction of the computation cost, results related to peak memory usage are provided in Appendix~\ref{app:memory_cost}.

\paragraph{Impact of the number of slices $L$}
The number of slices $L$ determines the accuracy--speed trade-off. Results in Appendix~\ref{app:pareto} show the Pareto front for
SinkSLOT at various $L$ alongside FlashSinkhorn. First, increasing $L$ improves the
achievable accuracy but increases computation time, consistent with
the $O(L(N+M))$ per-iteration cost. Second, for a given $L$,
decreasing $\varepsilon$ improves accuracy only up to a plateau that
reflects the gap between the support-restricted OT solution and the
true OT plan; only enriching the sparse support by increasing $L$ can lower this floor. Third, despite this trade-off,
SinkSLOT consistently reaches the same accuracy as FlashSinkhorn at
lower runtime for moderate gap levels. Very high precision ($\Delta \leq 1\%$) requires large $L$, but
such precision is rarely needed in practice due to the prohibitive
computational cost in this regime.

\paragraph{Scalability}
SinkSLOT's $O(L(N+M))$ per-iteration cost is asymptotically lower
than FlashSinkhorn's $O(NM)$ for fixed $L$.
Figure~\ref{fig:scalability} confirms this: SinkSLOT consistently
outperforms FlashSinkhorn across both $d{=}3$ and $d{=}64$, with speedups of $2$–$11\times$ at $L{=}4096$ and $4$–$36\times$ at $L{=}1024$. The advantage also persists across
dimensions (right panel): at fixed $N{=}10^4$, SinkSLOT remains faster
for all $d$ up to $1024$, even at $L{=}4096$.

\begin{figure*}[t]
\centering
\includegraphics[width=0.89\textwidth]{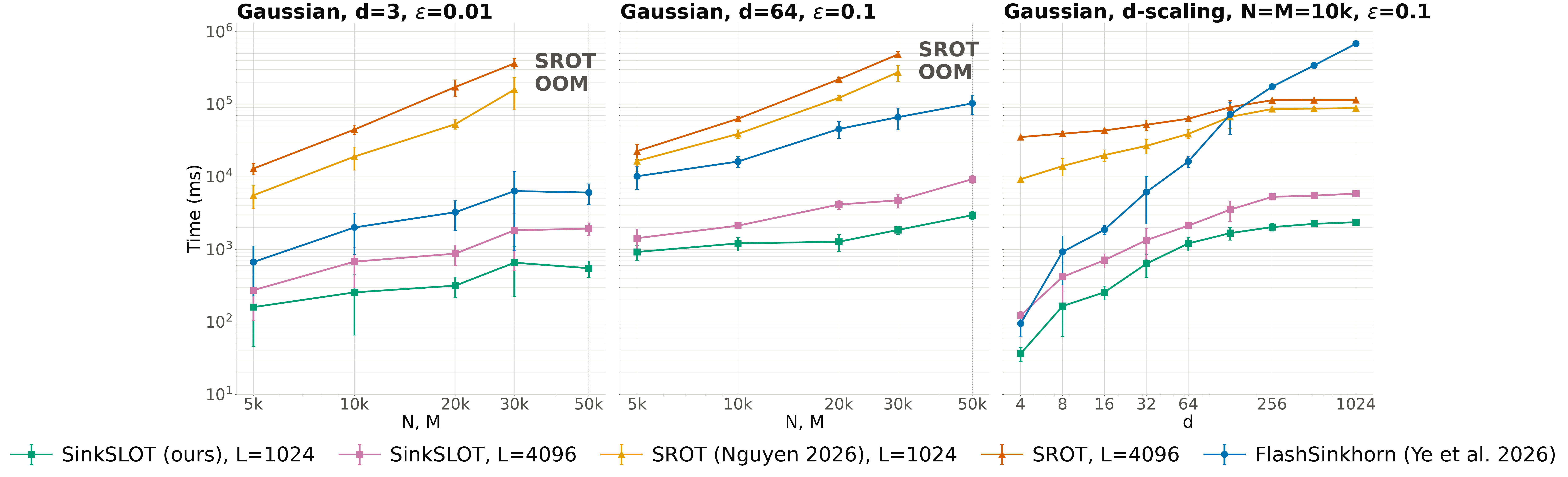}
\caption{Forward-pass runtime against $(N,M)$ and $d$. Error bars indicate the standard error over five random seeds.}
\label{fig:scalability}
\end{figure*}

\subsection{Gradient Flows}
 We
now assess whether the SinkSLOT divergence
$\mathrm{SLOT}_\varepsilon$~\eqref{eq:slot_val} is effective as a
loss function for optimisation. We run the blob-to-crescent gradient
flow experiment of EOT~\citep{feydy2019interpolating} and
SROT~\citep{nguyen2026sliced} with $N = M = 1000$, and additionally
report results for SOT~\citep{bonneel2015sliced} and
$\mathrm{SLOT}_\varepsilon$~\eqref{eq:slot_val}. At step~$t$, source points are updated:
$X^{(t+1)} = X^{(t)} - \eta\, n\, \nabla_X D(X^{(t)}, Y)$ with
$\eta = 0.05$. $D$~denotes either the \textit{Sliced Wasserstein Distance}~\citep{bonneel2015sliced} or the corresponding
divergence for EOT, SROT, and $\mathrm{SLOT}_\varepsilon$. All methods use the squared
Euclidean cost $C_{ij} = \|x_i - y_j\|^2$. For the divergences, we
set $\varepsilon = 0.01$. $L = 100$ projection vectors are sampled
whenever a sliced plan is involved (all except EOT),
and the SROT smoothing parameter~$\gamma$ in~\eqref{eq:srot} is set
to $10^{-8}$. We evaluate all methods using the exact squared
2-Wasserstein distance, $W_2^2$~\eqref{eq:ot}.

\begin{figure*}[t!]
\centering
\includegraphics[width=0.85\textwidth, trim=0 8 0 7, clip]{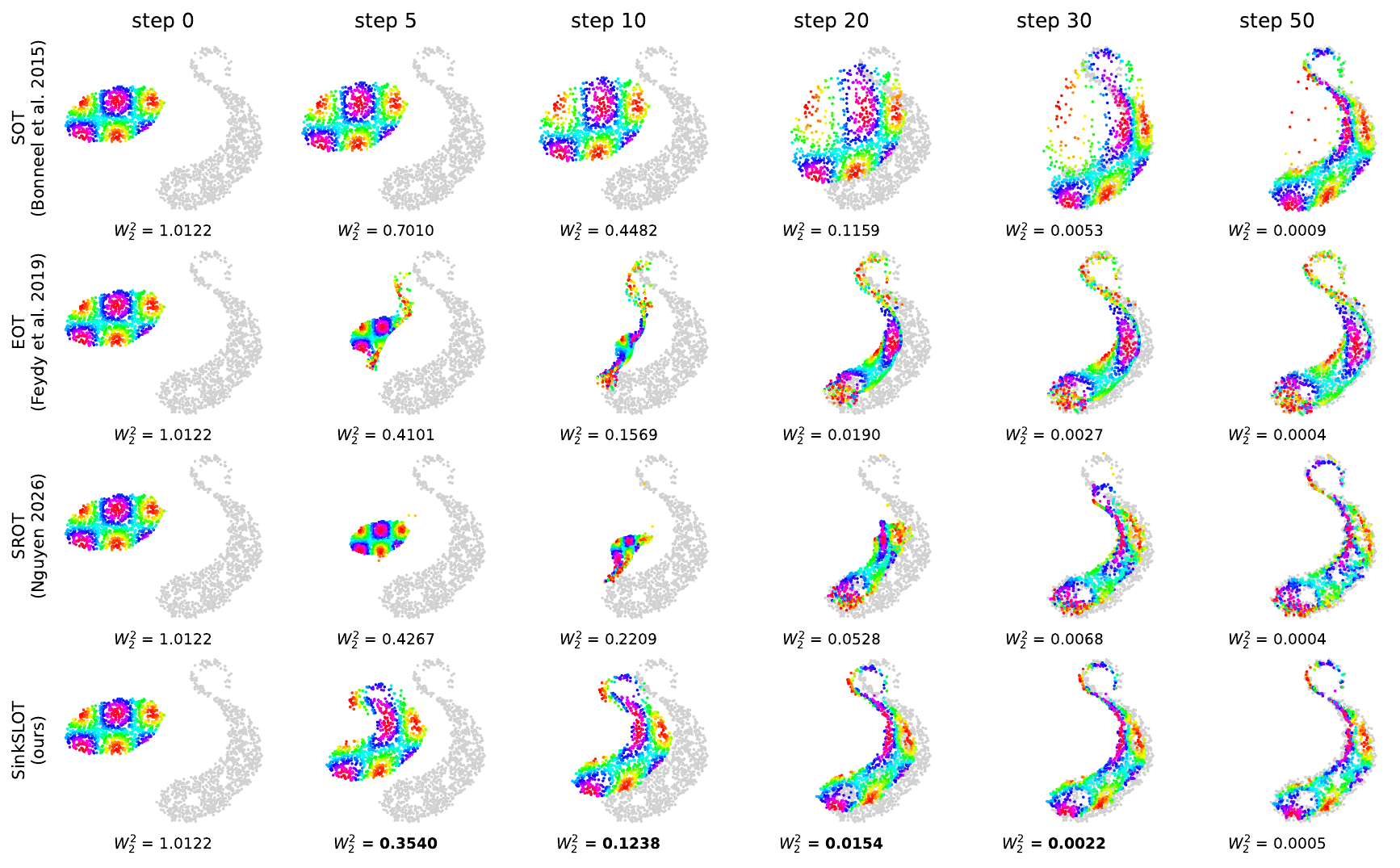}
\caption{Blob-to-crescent gradient flows. $N=M=1000$. Regularisation strength $\varepsilon=0.01$. $L=100$ slices. Performance is evaluated using the exact squared 2-Wasserstein distance $W_2^2$; \textbf{bold} indicates the lowest value at each intermediate step.}
\label{fig:grad_flow}
\end{figure*}

Figure~\ref{fig:grad_flow} shows the results after 50 Euler steps.
While all four methods attain similar final $W_2^2$ values, $\mathrm{SLOT}_\varepsilon$
transports the source points toward the target more rapidly during the
earlier steps. Quantitatively, $\mathrm{SLOT}_\varepsilon$ achieves the lowest $W_2^2$ and
lingers substantially below SROT through step~30. Qualitatively, at
step~5, $\mathrm{SLOT}_\varepsilon$ already rearranges the source cloud into the target
crescent shape, whereas the initial blob shape remains clearly visible
for the compared methods. The fact that $\mathrm{SLOT}_\varepsilon$ converges without
debiasing, while EOT requires the Sinkhorn divergence
correction~\citep{feydy2019interpolating}, is consistent with
Proposition~\ref{prop:divergence}. Overall, the informative gradients
and efficient forward computation make $\mathrm{SLOT}_\varepsilon$ divergence a promising
loss function for large-scale optimisation.

\section{Conclusion}
\label{sec:conclusion}
We introduced SinkSLOT, a sparse EOT method that replaces the dense
independent reference coupling with an expected lifted sliced OT
plan, reducing per-iteration cost from $O(NM)$ to $O(L(N+M))$. We
proved convergence of the sparse Sinkhorn iterations, characterised
limiting behaviour as $\varepsilon$ varies, and derived first- and
second-order derivatives. Experiments show that SinkSLOT achieves
accuracy comparable to dense and hardware-aware baselines while
substantially reducing computation time.

The number of slices $L$ is a tunable parameter: too few slices
restrict the transport polytope, while too many reduce the sparsity
advantage. Promising extensions include $\varepsilon$-scaling~\citep{kosowsky1994invisible},
non-uniform slice weighting~\citep{liu2025expected}, warm
starts~\citep{peyre2019computational}, and data-dependent projection
directions~\citep{nguyen2021distributional}.

Overall, experiments show that SinkSLOT delivers substantial speedups over state-of-the-art dense and sparse EOT methods
as a transport solver, and that the resulting divergence is a
promising loss function, as illustrated on a gradient flow task.

\section*{Acknowledgments}
This work was supported by ANR through the ``Investissements d’avenir'' program (ANR-10-IAIHU-06, ANR-19-P3IA-0001, PRAIRIE 3IA Institute), the ``France 2030'' program (ANR-23-IACL-0008, PRAIRIE-PSAI).  The ARAMIS Lab is affiliated with DIM C-BRAINS, funded by the Conseil Régional d’Ile-de-France. This work was performed using HPC resources from GENCI-IDRIS (Grant 2026-AD011017987). RD received Marie Skłodowska-Curie grant No. 101154248 (SafeREG). 
SSK acknowledges support from the UKRI MRC Doctoral Training Partnership at King's College London (MR/W006820/1). 
TV is co-founder and shareholder of Hypervision Surgical. 

\bibliographystyle{plainnat}
\bibliography{bibliography}

\clearpage
\appendix

\newcounter{myappendix}[section]
\renewcommand{\themyappendix}{\thesection.\arabic{myappendix}}

\newcommand{\appsubsection}[1]{%
  \refstepcounter{myappendix}%
  \subsection*{\themyappendix\quad #1}%
  \addcontentsline{toc}{subsection}{\themyappendix\quad #1}%
}

\newcounter{appsection}
\newcounter{appsubsection}[appsection]
\renewcommand{\theappsection}{\Alph{appsection}}
\renewcommand{\theappsubsection}{\theappsection.\arabic{appsubsection}}

\newcommand{\appsec}[1]{%
  \refstepcounter{appsection}%
  \section*{\theappsection.\quad #1}%
  \label{#1}%
}

\newcommand{\appsubsec}[1]{%
  \refstepcounter{appsubsection}%
  \subsection*{\theappsubsection\quad #1}%
}

\appsec{Proofs}

\appsubsec{Proof of Proposition~\ref{prop:unique}
  (unique diagonal scaling)}
\label{app:proof_unique}

\begin{proof}
Since $P^{\mathrm{SOT}}\in\Pi(\mathbf{a},\mathbf{b})$ is supported
on~$\mathcal{S}$, the feasible set
$\Pi(\mathbf{a},\mathbf{b})\cap
 \{P:\support(P)\subseteq\mathcal{S}\}$
is non-empty.
The KL divergence is strictly convex in its first argument over this
set, so the minimiser $P_\varepsilon^{\star}$ exists and is unique.

Writing the Lagrangian with multipliers $f_i$ (row marginals) and
$g_j$ (column marginals), the first-order condition for
$(i,j)\in\mathcal{S}$ gives
$\log\frac{P_{ij}}{(K_\varepsilon)_{ij}} - f_i - g_j = 0$,
hence $P_{ij} = e^{f_i}(K_\varepsilon)_{ij}\,e^{g_j}$.
Defining $u_i = e^{f_i}$ and $v_j = e^{g_j}$ yields the claimed
form; for $(i,j)\notin\mathcal{S}$, $P_{ij}^{\star}=0$ by the
support restriction.

We additionally verify scalability in the sense of
\citet{kalantari2008complexity}:
a nonnegative matrix $K$ is $(\mathbf{a},\mathbf{b})$-scalable if and
only if there exists a matrix with the same zero pattern satisfying
both marginal constraints.
Since $P^{\mathrm{SOT}}$ shares the zero pattern of $K_\varepsilon$
and lies in $\Pi(\mathbf{a},\mathbf{b})$, the condition is met.
\end{proof}

\appsubsec{Proof of Proposition~\ref{prop:sinkhorn}
  (convergence of sparse Sinkhorn)}
\label{app:proof_sinkhorn}

\begin{proof}
We apply \citet[Theorem~3.2]{csiszar1975divergence}.
Matrices $P\in\mathbb{R}_+^{N\times M}$ supported on $\mathcal{S}$
are identified with probability distributions on the finite set
$\mathcal{S}$,
and $K_\varepsilon$ plays the role of the reference measure~$R$ in
\citet{csiszar1975divergence}.
Define two linear constraint families
\[
  \mathcal{E}_r=\{P\ge0:P\mathbf{1}=\mathbf{a}\},\qquad
  \mathcal{E}_c=\{P\ge0:P^\top\mathbf{1}=\mathbf{b}\},
\]
so that $\Pi(\mathbf{a},\mathbf{b})=\mathcal{E}_r\cap\mathcal{E}_c$.
Each is a linear family in the sense of
\citet[Section~3, case~(A)]{csiszar1975divergence},
since the constraints are finite linear expectations.
The domain condition of \citet[Theorem~3.2]{csiszar1975divergence}
requires a distribution in
$\mathcal{E}_r\cap\mathcal{E}_c$ supported on~$\mathcal{S}$;
this is satisfied by~$P^{\mathrm{SOT}}$.

We identify each half-iteration of~\eqref{eq:sparse_sinkhorn}
with the I-projection (KL-projection) of the current iterate onto
$\mathcal{E}_r$ or~$\mathcal{E}_c$.
Starting from
$P^0=K_\varepsilon=\diag(\mathbf{1})\,K_\varepsilon\,\diag(\mathbf{1})$,
the I-projection of $P^\ell$ onto $\mathcal{E}_r$ preserves the
diagonal-scaling structure:
the first-order optimality conditions give
$\widetilde{P}^{\ell+1}_{ij}
 = u_i^{\ell+1}(K_\varepsilon)_{ij}v_j^\ell$,
and imposing $\sum_j\widetilde{P}^{\ell+1}_{ij}=a_i$ yields
$\mathbf{u}^{\ell+1}=\mathbf{a}\oslash(K_\varepsilon\mathbf{v}^\ell)$.
Similarly, projecting onto $\mathcal{E}_c$ gives
$\mathbf{v}^{\ell+1}
 =\mathbf{b}\oslash(K_\varepsilon^\top\mathbf{u}^{\ell+1})$.
By \citet[Theorem~3.2]{csiszar1975divergence},
the alternating I-projections converge to the I-projection of
$K_\varepsilon$ onto
$\mathcal{E}_r\cap\mathcal{E}_c=\Pi(\mathbf{a},\mathbf{b})$,
which is $P_\varepsilon^{\star}$.
\end{proof}

\appsubsec{Proof of Proposition~\ref{prop:divergence}
  ($\mathrm{SLOT}_\varepsilon$ is a divergence)}
\label{app:proof_divergence}

\begin{proof}\leavevmode
\paragraph{Non-negativity}
Since $C\ge0$ and $P_\varepsilon^{\star}\ge0$,
we have $\langle C,P_\varepsilon^{\star}\rangle\ge0$.
By Gibbs' inequality,
$\KL(P_\varepsilon^{\star}\|P^{\mathrm{SOT}})\ge0$.
Hence $\mathrm{SLOT}_\varepsilon(\mu,\nu)\ge0$.

\paragraph{Identity of indiscernibles}

$(\Leftarrow)$\;
When $\mu=\nu$, projecting identical point clouds onto any
direction~$\theta$ yields identical sorted positions, so the
rank-based 1D matching maps each point to itself.
Thus $P^{\mathrm{SOT}(\mu,\mu)}=\diag(\mathbf{a})$.
Since $C_{ii}=\|x_i-x_i\|^p=0$, the candidate plan
$P=\diag(\mathbf{a})$ achieves
$\langle C,\diag(\mathbf{a})\rangle=0$ and
$\KL\bigl(\diag(\mathbf{a})\|\diag(\mathbf{a})\bigr)=0$,
giving $\mathrm{SLOT}_\varepsilon(\mu,\mu)=0$.

$(\Rightarrow)$\;
If $\mathrm{SLOT}_\varepsilon(\mu,\nu)=0$, then both
$\langle C,P_\varepsilon^{\star}\rangle=0$ and
$\KL(P_\varepsilon^{\star}\|P^{\mathrm{SOT}})=0$.
The latter implies $P_\varepsilon^{\star}=P^{\mathrm{SOT}}$.
The former requires $C_{ij}=0$, hence $x_i=y_j$,
for all $(i,j)$ in the support of~$P^{\mathrm{SOT}}$.
Since $P^{\mathrm{SOT}}\in\Pi(\mathbf{a},\mathbf{b})$,
every source point $x_i$ sends mass only to target points
$y_j=x_i$, and every target point $y_j$ receives mass only from
source points $x_i=y_j$.
For any location $z\in\mathbb{R}^d$,
\begin{align*}
  \mu(\{z\})
  &= \sum_{i:x_i=z} a_i
   = \sum_{i:x_i=z}\sum_j P^{\mathrm{SOT}}_{ij}\\
  &= \sum_{i:x_i=z}\sum_{j:y_j=z} P^{\mathrm{SOT}}_{ij}
   = \sum_{j:y_j=z} b_j
   = \nu(\{z\}),
\end{align*}
so $\mu=\nu$.

\paragraph{Symmetry}
For any direction~$\theta$, the 1D rank-based optimal matching is
symmetric: swapping source and target reverses the matching
bijection.
Hence
$P^{\mathrm{SOT}(\nu,\mu)}
 =\bigl(P^{\mathrm{SOT}(\mu,\nu)}\bigr)^\top$.
Since $C_{ij}=\|x_i-y_j\|^p=\|y_j-x_i\|^p$,
the cost matrix for the reversed problem is~$C^\top$.
The Gibbs kernel for the reversed problem is therefore
$K_\varepsilon^\top$, and the optimal plan is
$(P_\varepsilon^{\star})^\top$.
Because KL is invariant under simultaneous transposition of both
arguments:
\begin{align*}
  \mathrm{SLOT}_\varepsilon(\nu,\mu)
  &= \langle C^\top,(P_\varepsilon^{\star})^\top\rangle
     +\varepsilon\,
      \KL\bigl(
        (P_\varepsilon^{\star})^\top
        \|(P^{\mathrm{SOT}})^\top
      \bigr)\\
  &= \mathrm{SLOT}_\varepsilon(\mu,\nu).
  \qedhere
\end{align*}
\end{proof}

\appsubsec{Proof of Proposition~\ref{prop:rate}
  (linear convergence rate)}
\label{app:proof_rate}

\begin{proof}
The argument combines the classical Jacobian analysis of Sinkhorn
iterates~\citep{soules1991rate,knight2008sinkhorn} with a connectivity
characterisation specific to the sparse support~$\mathcal{S}$.

\paragraph{Step 1: Jacobian and its spectrum}
Write one full Sinkhorn cycle as a map
$\Psi:\mathbb{R}^M\to\mathbb{R}^M$ on the log-scaling vectors
$\boldsymbol{\psi}=\log\mathbf{v}$
(cf.\ equations \eqref{eq:logrow}-\eqref{eq:logcol}).
A direct calculation using the softmax identity
$\partial_{x_k}\mathrm{LSE}(x)=e^{x_k}/\sum_l e^{x_l}$
shows that the Jacobian of $\Psi$ at the fixed point
$\boldsymbol{\psi}^*$ satisfies
\begin{equation}\label{eq:jacobian_entry}
  J_{jj'}
  =\sum_{i:\,(i,j),(i,j')\in\mathcal{S}}
   \frac{P^\star_{ij}\;P^\star_{ij'}}{a_i\;b_j}\,,
\end{equation}
i.e.\
$J = D_{\mathbf{b}}^{-1}(P_\varepsilon^\star)^\top
     D_{\mathbf{a}}^{-1}P_\varepsilon^\star$;
see \citet[Theorem~3.1]{knight2008sinkhorn} for the general
(possibly sparse) case.
The similarity transformation
$\widetilde{J}
 = D_{\mathbf{b}}^{1/2}\,J\,D_{\mathbf{b}}^{-1/2}
 = A^\top A$
shows that the eigenvalues of $J$ are the squared singular values
$\sigma_1^2\geq\sigma_2^2\geq\cdots$ of the normalised plan
matrix~$A$ defined in~\eqref{eq:normalised_plan}.

\paragraph{Step 2: Leading singular value}
A direct check confirms $\sigma_1=1$:
\begin{align*}
  A\,\sqrt{\mathbf{b}}
  &= D_{\mathbf{a}}^{-1/2}\,P_\varepsilon^\star\,\mathbf{1}
   = D_{\mathbf{a}}^{-1/2}\,\mathbf{a}
   = \sqrt{\mathbf{a}},\\
  A^\top\!\sqrt{\mathbf{a}}
  &= D_{\mathbf{b}}^{-1/2}\,(P_\varepsilon^\star)^\top\mathbf{1}
   = \sqrt{\mathbf{b}}.
\end{align*}
The corresponding eigenvector of $J$ is
$D_{\mathbf{b}}^{-1/2}\sqrt{\mathbf{b}}=\mathbf{1}$, which spans
the gauge direction
$\boldsymbol{\psi}\mapsto\boldsymbol{\psi}+c\,\mathbf{1}$.

\paragraph{Step 3: Connectivity implies $\sigma_2<1$}
The entry $(A^\top A)_{jj'}$ is positive if and only if columns $j$
and $j'$ share a row neighbour in~$G_\mathcal{S}$.
This adjacency defines the \emph{column intersection graph}~$H$.
When $G_\mathcal{S}$ is connected, so is~$H$: any path
$c_j\!-\!r_{i_1}\!-\!c_{j_1}\!-\!r_{i_2}\!-\!\cdots\!-\!c_{j'}$
in $G_\mathcal{S}$ maps to a walk
$j\!-\!j_1\!-\!\cdots\!-\!j'$ in~$H$,
since each consecutive column pair shares a row neighbour.
Hence $A^\top\!A$ is irreducible, and the Perron-Frobenius theorem
implies that its leading eigenvalue~$1$ is simple, giving
$\sigma_2^2<1$.
Conversely, if $G_\mathcal{S}$ has $q\geq2$ connected components,
$A$ decomposes into $q$ independent blocks each contributing a unit
singular value, so $\sigma_2=1$.

\paragraph{Step 4: Linear convergence of the coupling}
The log-domain map $\Psi$ satisfies the gauge equivariance
$\Psi(\boldsymbol{\psi}+c\mathbf{1})
 =\Psi(\boldsymbol{\psi})+c\mathbf{1}$,
so the eigenvalue $1$ of~$J$ corresponds to the gauge direction and
does not affect the coupling
$P^\ell=\diag(e^{\boldsymbol{\varphi}^\ell})\,K_\varepsilon\,
 \diag(e^{\boldsymbol{\psi}^\ell})$.
On the quotient space $\mathbb{R}^M/\mathrm{span}(\mathbf{1})$
the induced map has a unique fixed point whose Jacobian has spectral
radius $\sigma_2^2<1$.
By the standard convergence theorem for differentiable fixed-point
iterations~\citep[Theorem~4.1]{ortega2000iterative},
\[
  \limsup_{\ell\to\infty}\;
  \bigl\|P^\ell - P_\varepsilon^{\star}\bigr\|_1^{\,1/\ell}
  \;\leq\;\sigma_2^{\,2}.
\]
The iteration count~\eqref{eq:iter_count} and the per-iteration cost
$O(|\mathcal{S}|)=O(L(N{+}M))$ follow immediately.
\end{proof}

\appsubsec{Proof of Proposition~\ref{prop:limits}
  (limiting behaviour)}
\label{app:proof_limits}

\begin{proof}\leavevmode

\paragraph{Part 1: $\varepsilon\to0^+$}
We first note that $\KL(P\|P^{\mathrm{SOT}})$ is finite and
continuous on $\Pi_\mathcal{S}$:
for any $(i,j)\in\mathcal{S}$ we have $P^{\mathrm{SOT}}_{ij}>0$,
and the map $x\mapsto x\log(x/q)$ is continuous on $[0,\infty)$
for $q>0$ (with the convention $0\log0=0$).
Since $\Pi_\mathcal{S}$ is compact (a closed bounded subset of
$\mathbb{R}^{|\mathcal{S}|}$), the KL divergence is bounded:
$\sup_{P\in\Pi_\mathcal{S}}\KL(P\|P^{\mathrm{SOT}})
 =:K_{\max}<\infty$.

Let $\varepsilon_k\to0^+$ and write $P_k=P_{\varepsilon_k}^\star$.
By compactness, $(P_k)$ has a convergent subsequence (still
denoted~$P_k$) with $P_k\to\bar{P}\in\Pi_\mathcal{S}$.

\emph{$\bar{P}$ is a cost minimiser.}\;
The optimality of~$P_k$ yields, for any $P\in\Pi_\mathcal{S}$:
\begin{equation}\label{eq:opt_eps_app}
  \langle C,P_k\rangle
  +\varepsilon_k\,\KL(P_k\|P^{\mathrm{SOT}})
  \;\leq\;
  \langle C,P\rangle
  +\varepsilon_k\,\KL(P\|P^{\mathrm{SOT}}).
\end{equation}
Since $\KL(P_k\|P^{\mathrm{SOT}})\geq0$:
\[
  \langle C,P_k\rangle\;\leq\;\langle C,P\rangle+\varepsilon_k K_{\max}.
\]
Taking $k\to\infty$:
$\langle C,\bar{P}\rangle\leq\langle C,P\rangle$
for all $P\in\Pi_\mathcal{S}$.
Hence $\bar{P}\in\mathcal{P}^*
:=\argmin_{P\in\Pi_\mathcal{S}}\langle C,P\rangle$.

\emph{$\bar{P}$ has minimum KL among cost minimisers.}\;
Let $\mathrm{OT}_\mathcal{S}=\langle C,\bar{P}\rangle$ and let
$\widehat{P}=\argmin\{
  \KL(P\|P^{\mathrm{SOT}}):P\in\mathcal{P}^*\}$
(unique by strict convexity of KL and convexity of~$\mathcal{P}^*$).
Setting $P=\widehat{P}$ in~\eqref{eq:opt_eps_app} and using
$\langle C,P_k\rangle\geq\mathrm{OT}_\mathcal{S}$:
\begin{align}
    \KL(P_k\|P^{\mathrm{SOT}})
  \;&\leq\;
  \KL(\widehat{P}\|P^{\mathrm{SOT}})
  +\frac{\langle C,\widehat{P}\rangle-\langle C,P_k\rangle}
        {\varepsilon_k}
  \\
  &\leq\;
  \KL(\widehat{P}\|P^{\mathrm{SOT}}),
\end{align}
where the last inequality uses
$\langle C,P_k\rangle\geq\langle C,\widehat{P}\rangle
 =\mathrm{OT}_\mathcal{S}$.
By continuity of KL on $\Pi_\mathcal{S}$:
\[
  \KL(\bar{P}\|P^{\mathrm{SOT}})
  \leq\liminf_k\,\KL(P_k\|P^{\mathrm{SOT}})
  \leq\KL(\widehat{P}\|P^{\mathrm{SOT}}).
\]
Since $\bar{P}\in\mathcal{P}^*$ and $\widehat{P}$ minimises KL
over~$\mathcal{P}^*$, we also have
$\KL(\bar{P}\|P^{\mathrm{SOT}})
 \geq\KL(\widehat{P}\|P^{\mathrm{SOT}})$.
Hence equality holds, and by the uniqueness of the KL minimiser
over the convex set~$\mathcal{P}^*$,
$\bar{P}=\widehat{P}$.
Since every convergent subsequence has the same limit, the entire
family converges:
$P_\varepsilon^\star\to\widehat{P}$ as $\varepsilon\to0^+$.

\paragraph{Part 2: $\varepsilon\to+\infty$}
Dividing the optimality condition~\eqref{eq:opt_eps_app}
by~$\varepsilon$ and setting $P=P^{\mathrm{SOT}}$
(for which $\KL(P^{\mathrm{SOT}}\|P^{\mathrm{SOT}})=0$):
\[
  \frac{\langle C,P_\varepsilon^\star\rangle}{\varepsilon}
  +\KL(P_\varepsilon^\star\|P^{\mathrm{SOT}})
  \;\leq\;
  \frac{\langle C,P^{\mathrm{SOT}}\rangle}{\varepsilon}.
\]
Since $\langle C,P_\varepsilon^\star\rangle\geq0$:
\[
  \KL(P_\varepsilon^\star\|P^{\mathrm{SOT}})
  \;\leq\;
  \frac{\langle C,P^{\mathrm{SOT}}\rangle}{\varepsilon}.
\]
By Pinsker's inequality,
$\|P-Q\|_1\leq\sqrt{2\,\KL(P\|Q)}$, so
\[
  \|P_\varepsilon^\star-P^{\mathrm{SOT}}\|_1
  \;\leq\;
  \sqrt{\frac{2\,\langle C,P^{\mathrm{SOT}}\rangle}{\varepsilon}}
  \;\xrightarrow{\;\varepsilon\to\infty\;}\;0.
  \qedhere
\]
\end{proof}

\appsubsec{Proof of Proposition~\ref{prop:gradient}
  (gradient of $\mathrm{SLOT}_\varepsilon$)}
\label{app:proof_gradient}

\begin{proof}
By Danskin's theorem~\citep{bertsekas1997nonlinear},
since $P^\star$ is the unique minimiser of a strictly convex
programme in~$P$ (Proposition~\ref{prop:unique}),
\[
  \frac{d\,\mathrm{SLOT}_\varepsilon}{dx_{k,t}}
  = \sum_j P^\star_{kj}\,\frac{\partial C_{kj}}{\partial x_{k,t}}
  = \sum_{j\in\mathcal{S}_k}P^\star_{kj}\cdot 2(x_{k,t}-y_{j,t}),
\]
where the sum is restricted to $\mathcal{S}_k$ because
$P^\star_{kj}=0$ for $(k,j)\notin\mathcal{S}$.
Using $\sum_j P^\star_{kj}=a_k$ yields the barycentric
form~\eqref{eq:grad_slot}.
\end{proof}

\appsubsec{Proof of Proposition~\ref{prop:hessian}
  (Hessian of $\mathrm{SLOT}_\varepsilon$)}

\begin{definition}[Sensitivity matrix]\label{def:sensitivity}
Define the \emph{sensitivity matrix}
\begin{equation}\label{eq:sensitivity}
  H^\star
  =\begin{pmatrix}
     \diag(\mathbf{a}) & P^\star\\
     (P^\star)^\top     & \diag(\mathbf{b})
   \end{pmatrix}
  \in\mathbb{R}^{(N+M)\times(N+M)}.
\end{equation}
Under assumption that the bipartite support graph $G_\mathcal{S}$ is connected,
$H^\star$ is symmetric positive semidefinite with a simple zero
eigenvalue (null vector $(\mathbf{1}_N,-\mathbf{1}_M)$),
and its Moore-Penrose pseudoinverse $H^{\star\dagger}$ is
well-defined.
\end{definition}

\begin{remark}\label{rem:sparse_sensitivity}
In the dense EOT case, $H^\star$ is a dense
$(N{+}M)\times(N{+}M)$ matrix because $P^\star>0$ everywhere.
In SinkSLOT, $H^\star$ inherits the sparsity of~$P^\star$:
its off-diagonal blocks have $|\mathcal{S}|$ nonzeros,
so the matrix-vector product $H^\star w$ costs
$O(|\mathcal{S}|+N+M)$ rather than $O(NM)$.
\end{remark}

\begin{definition}[Auxiliary tensors]\label{def:aux_tensors}
Define $\mathcal{B}\in\mathbb{R}^{N\times d\times M}$ by
$\mathcal{B}_{ktj}=2(x_{k,t}-y_{j,t})\,P^\star_{kj}$,
and $\mathcal{R}\in\mathbb{R}^{(N+M)\times N\times d}$ slice-wise
(for each $t\in[d]$) as
\[
  \mathcal{R}_{:,:,t}
  =\begin{pmatrix}
     \diag\!\bigl(\mathcal{B}_{:,t,:}\,\mathbf{1}_M\bigr)\\[2pt]
     (\mathcal{B}_{:,t,:})^\top
   \end{pmatrix}.
\]
\end{definition}

\begin{proposition}[Hessian of SinkSLOT]\label{prop:hessian}
Under the stop-gradient convention and squared Euclidean cost,
the Hessian
$\mathcal{T}=\nabla^2_X\mathrm{SLOT}_\varepsilon
 \in\mathbb{R}^{N\times d\times N\times d}$
decomposes as
\begin{equation}\label{eq:hessian_decomp}
  \mathcal{T}_{ktsl}
  =\frac{1}{\varepsilon}\sum_{i,j=1}^{N+M}
   \mathcal{R}_{ikt}\,H^{\star\dagger}_{ij}\,\mathcal{R}_{jsl}
  \;+\;\mathcal{E}_{ktsl},
\end{equation}
where $\mathcal{E}$ is block-diagonal across source points:
$\mathcal{E}_{k,:,s,:}=0$ for $k\neq s$, and
\begin{equation}\label{eq:explicit_term}
  \mathcal{E}_{k,:,k,:}
  = 2a_k\,\mathbb{I}_d
    -\frac{4}{\varepsilon}\sum_{j\in\mathcal{S}_k}
     P^\star_{kj}\,(x_k-y_j)(x_k-y_j)^\top,
\end{equation}
where each point $x_k, y_j\in\mathbb{R}^{d\times 1}$ is treated
as a column vector, so
$(x_k-y_j)(x_k-y_j)^\top\in\mathbb{R}^{d\times d}$ is the
outer product.
\end{proposition}

\begin{proof}\label{app:proof_hessian}
Under the stop-gradient convention,
$\partial P^\star_{ij}/\partial x_{k,t}
 = (P^\star_{ij}/\varepsilon)
   (\partial f_i/\partial x_{k,t}
   +\partial g_j/\partial x_{k,t}
   -\partial C_{ij}/\partial x_{k,t})$,
since the only $x$-dependence in
$\log P^\star_{ij}
 =\log P^{\mathrm{SOT}}_{ij}+(f_i+g_j-C_{ij})/\varepsilon$
comes from $C_{ij}$ and the dual potentials $(f,g)$
(with $P^{\mathrm{SOT}}$ treated as constant).
This is the same structure as the dense EOT case
(where $\log P^\star_{ij}
 =\log a_i+\log b_j+(f_i+g_j-C_{ij})/\varepsilon$
and $\log a_i,\log b_j$ are constant),
so the remainder of the derivation (differentiating the marginal
constraints, solving the linear system via $H^{\star\dagger}$,
and collecting terms) is identical to
\citet[Theorem~7 / Appendix~C]{ye2026flashsinkhorn}.
\end{proof}

\appsec{Implementation details}\label{app:implementation_details}
\appsubsec{Synthetic Benchmark Experiments}
\label{app:benchmark_exp}

\paragraph{Shared protocol}
All methods solve the same problem instances and are held to the same
stopping rule, so runtimes are directly comparable.

Except in the scalability experiment in Figure~\ref{fig:scalability}, each dataset contains $N=M=10^4$ points. The weights are drawn uniformly from $[0.1,1.1]$ and normalised to sum to one. The three 2D datasets are visualised in Figure~\ref{fig:data_visualisation}.

\begin{figure*}
    \centering
\includegraphics[width=0.89\textwidth]{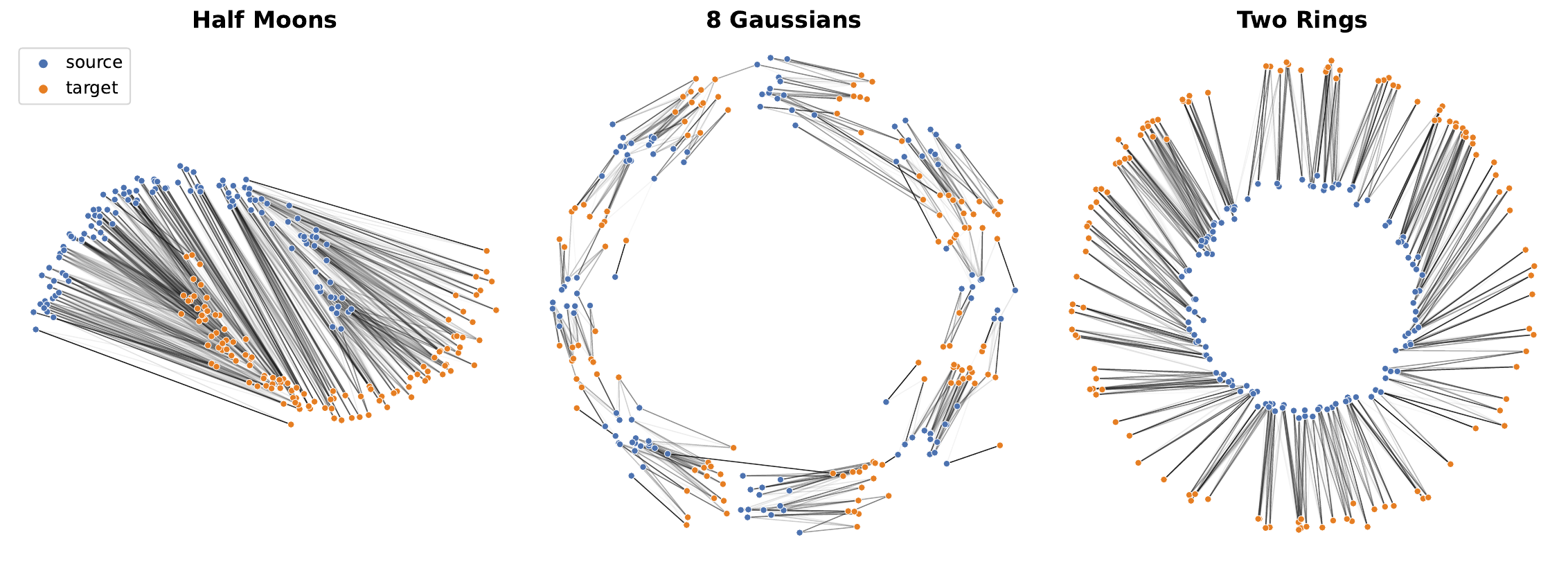}
\caption{Visualisation of SinkSLOT transport plan for the three 2D datasets. Only $120$ of the $10^4$ points are shown for readability. Regularisation strength $\varepsilon=0.01$. $L=1024$ slices.}
\label{fig:data_visualisation}
\end{figure*}

All methods are run in \texttt{float32} on the same NVIDIA H100-80GB GPU. Iterations terminate when the maximum marginal violation satisfies
\[
\max\Bigl(
\max_i \bigl|(P\mathbf{1})_i-a_i\bigr|,
\max_j \bigl|(P^\top\mathbf{1})_j-b_j\bigr|
\Bigr) < 10^{-6},
\]
or when the common limit of $10^4$ Sinkhorn iterations is reached. This limit is increased to $2\times10^4$ for the Gaussian dataset with $d=64$ and for the scalability experiment.

Every output plan is subject to a feasibility check, including two criteria: (i) mass conservation, $|\sum_{ij}P_{ij}-1|<10^{-3}$ and (ii) a marginal violation below $10^{-4}$. All $840$ output plans produced in the benchmark satisfy both criteria.

\paragraph{Hyperparameters}
We sweep eight values of the entropic regularisation parameter $\varepsilon$ for every method, eight slice counts $L$ for SROT and SinkSLOT, and four sampling budgets $s$ for Spar-Sink.

Since the effective entropic smoothing is governed by the ratio $C_{ij}/\varepsilon$, $\varepsilon$ should be selected relative to the scale of the ground cost~\citep{peyre2019computational}. We therefore sample eight log-spaced values in $[0.001,0.1]$ for the low-dimensional datasets and in $[0.1,1]$ for the Gaussian dataset with $d=64$.

For the slice-based methods, we use $L\in\{32,64,128,\ldots,4096\}$ on the low-dimensional datasets and $2L$ on Gaussian $d=64$. For Spar-Sink, we follow the experimental set of \citet{li2023importance} and use $s\in\{5,10,15,20\}\times s_0(N)$, where $s_0(N)=10^{-3}N\log^4(N)$. For $N=M=10^4$, these values correspond to expected sampling densities $s/N^2$ of approximately ${0.36\%,0.72\%,1.08\%,1.44\%}$.

For each method and cost-gap threshold, Table~\ref{tab:accuracy_speed} reports the minimum runtime achieved across the hyperparameter combinations considered. The reference plan $P^{\mathrm{OT}}$ is the exact solution of the
unregularised linear programme.

\paragraph{Baseline implementation}
SinkSLOT's implementation is described in Appendix~\ref{app:cuda}. For the baselines, we start from the authors'
released code and adapt it to our benchmarking stack.
The released implementations target a different software stack and
report different stopping and timing conventions, so running them as
published would confound the method with its implementation.
Adapting both holds the GPU, precision, stopping rule and timing
instrumentation fixed across all methods, leaving the algorithm as the
only difference.
The changes are confined to that harness: the update equations, the
reference coupling and the sampling scheme are the authors' own, and we
link the original code in each case.

\paragraph{FlashSinkhorn}
We use the authors' official FlashSinkhorn implementation\footnote{\url{https://github.com/ot-triton-lab/flash-sinkhorn}}.
Its solver is left intact apart from the instrumentation required to apply the shared marginal-violation stopping criterion instead of its native potential-change criterion. We use the alternating backend, rather than the symmetric backend, which updates both dual potentials simultaneously using half-steps. We set \texttt{allow\_tf32=False} to ensure that computations are performed in \texttt{float32}, consistent with all other methods.

\paragraph{SROT}
Adapted from the author's code~\citep{nguyen2026sliced}\footnote{\url{https://github.com/khainb/SROT}}.
The reference plan is the smoothed sliced plan
$P^{\mathrm{SOT}}_\gamma=(1-\gamma)P^{\mathrm{SOT}}+\gamma\,\mathbf{a}\otimes\mathbf{b}$
of~\eqref{eq:srot}. Consistent with the author's choice, we set $\gamma=10^{-8}$.

\paragraph{Spar-Sink}
Adapted from the authors' code~\citep{li2023importance}\footnote{\url{https://github.com/Mengyu8042/Spar-Sink}}.
Following their equations~(7) and~(9), entries are kept independently
with inclusion probability $q_{ij}=\min(1,\,s\,p_{ij})$ where
$p_{ij}\propto\sqrt{a_ib_j}$, and survivors are rescaled by $1/q_{ij}$ so
that the sampled kernel is an unbiased estimator for $K$.
The budget $s$ bounds the \emph{expected} number of nonzero entries, so the
realised count varies across runs. The values of $s$ included in the sweep are specified above.

\appsubsec{Performance Improvements from Fusion Strategies}
\label{app:cuda}

Table~\ref{tab:complexity} shows that the number of arithmetic operations per Sinkhorn iteration scales as $O(|\mathcal{S}|)$. In practice, however, wall-clock time scales as $|\mathcal{S}|$ times a constant factor that depends on implementation details, and a naive implementation can inflate that constant well beyond what the asymptotic bound suggests. For this reason, we wrote our kernels in Triton, taking inspiration from FlashSinkhorn's approach.
They read the support $\mathcal{S}$ through a compressed sparse layout we
build ourselves, rather than through a general-purpose sparse library.
The reason is that the row and column half-steps are segmented
log-sum-exp reductions, not matrix-vector products.
No sparse Basic Linear Algebra Subprograms (BLAS) library exposes that
primitive.
Two fusions carry the implementation, one in each stage of the pipeline.

\paragraph{Fusing the sparse cost evaluation}
Stage~1 needs $C_{ij}$ only on $\mathcal{S}$.
Written directly,
$(X[\texttt{rows}]-Y[\texttt{cols}])^2.\mathrm{sum}(1)$
allocates three $(|\mathcal{S}|,d)$ intermediates before reducing them.
We fuse the gather, difference, square and reduction into one Triton pass
that accumulates over $d$ in registers and writes only the $|\mathcal{S}|$
costs, evaluating $26$M entries in 3.1~ms and matching the dense
computation to a relative error of $2.8\times10^{-7}$.

\paragraph{Fusing the segmented log-sum-exp}
Each half-step of~\eqref{eq:logrow}-\eqref{eq:logcol} is a segmented LSE
over $\mathcal{S}$.
We store the row half-step in CSR and the column half-step in CSC so both
reductions are contiguous, and fuse each into a kernel that sweeps a row
in tiles of width $B$, carrying a running maximum $m$ and accumulator $s$. This is the streaming log-sum-exp recurrence of~\citet{nowozin2016streaming}, in the
generalised, parallel-reduction form of~\citet{milakov2018online}:
\[
  m'=\max\bigl(m,\textstyle\max_t x_t\bigr),
  \qquad
  s'=s\,e^{m-m'}+\textstyle\sum_t e^{x_t-m'},
\]
where $x_t$ enumerates $\log(K_\varepsilon)_{ij}+\psi_j$ over the tile's entries
$j$.
This never materialises those values and reads each entry once per
half-step rather than twice~\citep{ye2026flashsinkhorn}.
This is the same online-softmax trick FlashAttention~\citep{dao2022flashattention}
uses to fuse attention without materialising the score matrix, and that
FlashSinkhorn~\citep{ye2026flashsinkhorn} uses for its own dense Sinkhorn
kernels; we apply it to the segmented, per-row case the sparse support requires.
Folding $\log a_i-\mathrm{LSE}$ into the same kernel halves the launches
per iteration, from four to two.
One program owns one row and the tile width is set to
$B=2^{\lceil\log_2(|\mathcal{S}|/N)\rceil}$, the mean row length rounded
up to a power of two and clamped to $[32,1024]$.

\appsubsec{Gradient Flows}
\label{app:grad_flow}

We follow the convention of~\citet{nguyen2026sliced,feydy2019interpolating} as closely as possible. $N=M=1000$ uniformly weighted source and target points are sampled from two density images (blob $\rightarrow$ crescent), using density-weighted pixel sampling with sub-pixel jitter, $\pm (0.5/h)\times\text{noise}$, where $h$ is the image height in pixels, to avoid samples lying exactly on the discrete pixel grid. We then perform $50$ Euler steps, as described in the main text. The four methods are evaluated using the exact squared $2$-Wasserstein distance $W_2^2$:
\begin{itemize}
    \item \textbf{SOT:} We compute the exact analytical gradient using automatic differentiation through \texttt{ot.\allowbreak sliced\_\allowbreak wasserstein\_\allowbreak distance} from the POT library~\citep{flamary2021pot}, using $L=100$ slices. Since 1D OT has a closed-form solution based on sorting, no Sinkhorn iterations or entropic approximation are involved, and the results are therefore independent of $\varepsilon$.

    \item \textbf{EOT:} Rather than applying the marginal-violation stopping criterion used in the benchmark experiments, we run dense Sinkhorn for a fixed $1000$ iterations and use full backpropagation to compute the gradients. We set $\varepsilon=0.01$.

    \item \textbf{SROT:} In addition to the EOT setup, we use $L=100$ slices and set the smoothing parameter $\gamma$ in~\eqref{eq:srot} to $10^{-8}$, as proposed in~\citet{nguyen2026sliced}. As in our approach, the authors' implementation treats the prior coupling as a stop-gradient constant.

    \item \textbf{SinkSLOT (ours):} No backpropagation is required; analytical gradients are computed using Equation~\eqref{eq:grad_slot}. We run $1000$ fixed iterations with $\varepsilon=0.01$ and $L=100$.
\end{itemize}

In Appendix~\ref{app:stop_gradient}, we validate the substitution of our analytical gradient formula~\eqref{eq:grad_slot} for full backpropagation through $P^{\mathrm{SOT}}$, in which the prior coupling is treated as differentiable with respect to the source points.

\appsec{Additional results}\label{app:additional_results}

\begin{figure*}[t]
\centering
\includegraphics[width=0.89\textwidth]{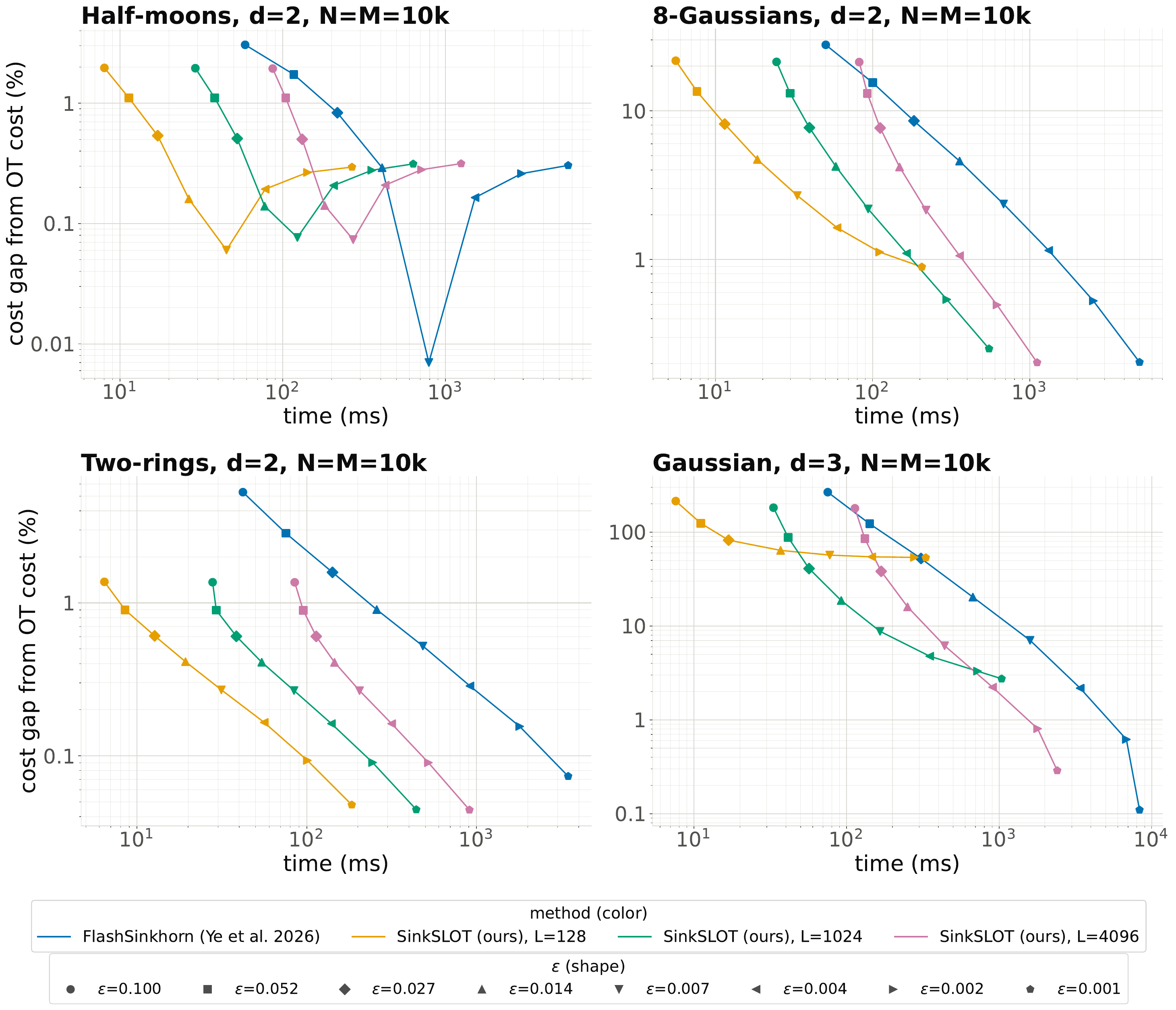}
\caption{Pareto fronts for the accuracy-speed trade-off on the four low-dimensional datasets. FlashSinkhorn vs SinkSLOT at three slice counts: $L=128$, $1024$, and $4096$. $N=M=10^4$.}
\label{fig:pareto}
\end{figure*}

\begin{figure}[t]
\centering
\includegraphics[width=0.45\textwidth]{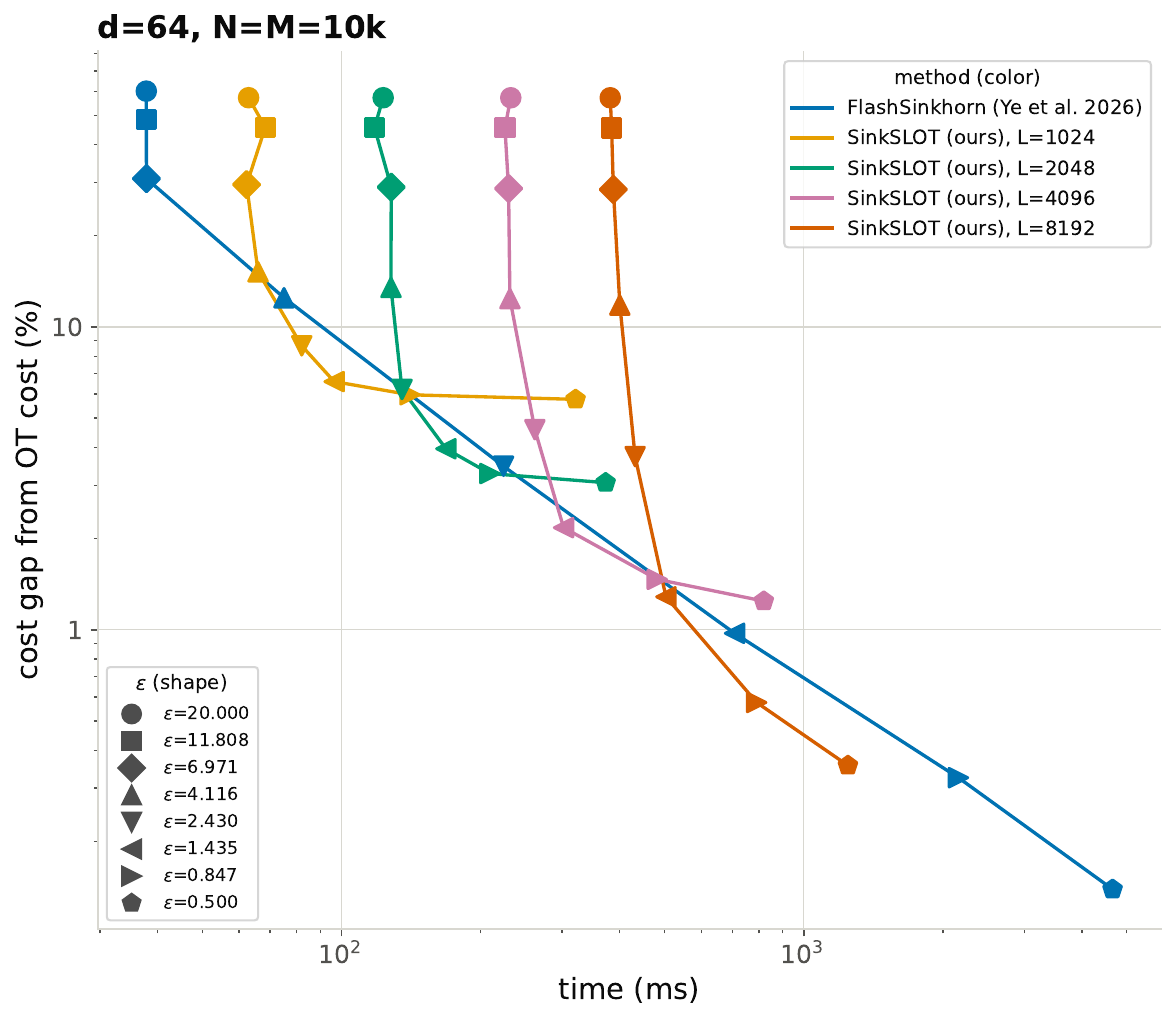}
\caption{Pareto fronts for the accuracy--speed trade-off on the Gaussian $d=64$ dataset. FlashSinkhorn vs SinkSLOT at four slice counts.}
\label{fig:pareto_gauss64}
\end{figure}

\appsubsec{Impact of L}\label{app:pareto}
To understand the accuracy-speed trade-off one faces when varying regularisation strength $\varepsilon$ and slice count $L$, Figure~\ref{fig:pareto} plots the cost gap from exact OT cost against runtime for FlashSinkhorn and SinkSLOT on each of the four low-dimensional datasets. The plot uses the eight values of $\varepsilon$ swept in the benchmark experiment and three slice counts: $128$ (low), $1024$ (medium), $4096$ (high). Since the results needed for this plot are readily available from the benchmark experiment, no additional runs were performed. For the Gaussian $d=64$ dataset, we run additional experiments using larger $\varepsilon$ to show the trade-off at higher cost gaps in Figure~\ref{fig:pareto_gauss64}.

Figure~\ref{fig:pareto} shows that, in low-dimensional settings, SinkSLOT achieves a better accuracy-speed trade-off than FlashSinkhorn. In other words, given a target cost gap, we can find a combination of hyperparameters $(L,\varepsilon)$ for which SinkSLOT is faster. Similarly, given a runtime budget, we can find $(L,\varepsilon)$ for which SinkSLOT achieves a lower cost gap. 

As with many optimisation methods, the question is then how to select $(L,\varepsilon)$ to achieve the optimal trade-off. There is no simple answer, and the appropriate choice depends strongly on the dataset. For \textit{two-rings}, a low slice count of $L=128$ appears sufficient, as its curve lies entirely southwest of the other three curves. Decreasing $\varepsilon$ reduces the cost gap while increasing runtime. For the other three datasets, the overall Pareto front for SinkSLOT comprises points obtained using multiple slice counts. 

The Gaussian dataset with $d=3$ illustrates this behaviour particularly well. At $L=128$, runtime is low, and decreasing $\varepsilon$ initially reduces the cost gap. Beyond a certain point, however, further reduction in $\varepsilon$ increases runtime without improving the quality of the output plan. In contrast to \textit{two-rings}, increasing the slice count to $L=1024$ is more effective. At large $\varepsilon$, the cost gap is comparable to $L=128$, but decreasing $\varepsilon$ has a greater effect on the cost gap at this slice count. The same reasoning applies when increasing the slice count from $L=1024$ to $L=4096$ when the objective is to achieve a cost gap below $1\%$.

The same conclusion holds in the high-dimensional setting shown in Figure~\ref{fig:pareto_gauss64}, even though the advantage over FlashSinkhorn is less pronounced. For a given cost-gap-runtime trade-off achieved by FlashSinkhorn, meticulously selected combinations $(L,\varepsilon)$ can achieve both a lower cost gap and a shorter runtime, but this set is likely to be small. In practice, the gains in accuracy and speed over FlashSinkhorn may not compensate for the time investment on experimentation to identify an appropriate hyperparameter range for a given distribution. This limitation is not unique to SinkSLOT. Rather, it is a broader challenge for slice-based OT methods, as the approximation ability of lifted plans deteriorates when dimension increases~\citep{tanguy2025sliced}.

\appsubsec{Convergence speed}\label{app:convergence_speed}
\input{Tables/convergence}
Table~\ref{tab:accuracy_iters} reports the number of iterations required for convergence for each fastest run reported in Table~\ref{tab:accuracy_speed}. No additional runs were performed to produce this table. Here, one iteration consists of sequentially updating the two dual potentials, corresponding to Steps~11 and~12 of Algorithm~\ref{alg:sinkslot}. 

In all selected runs except one (SinkSLOT on Gaussian $d=3$),  Sinkhorn converged within $10^4$ iterations. SinkSLOT required a number of iterations comparable to that of its dense counterpart, SROT.

\appsubsec{Memory cost}\label{app:memory_cost}
\input{Tables/memory}

\begin{figure*}[t]
\centering
\includegraphics[width=0.89\textwidth]{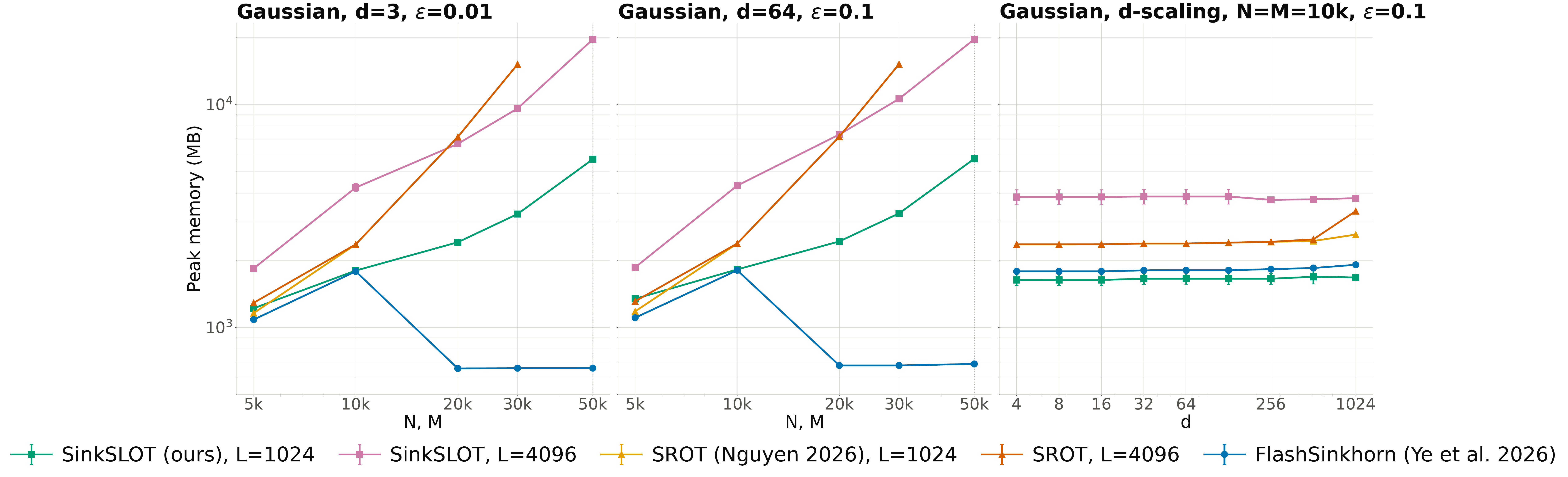}
\caption{Forward-pass peak memory against $(N,M)$ and $d$. Error bars indicate the standard error over five random seeds.}
\label{fig:scalability_memory}
\end{figure*}

Table~\ref{tab:accuracy_speed} and Figure~\ref{fig:scalability} report the shortest runtime achievable at each predefined cost gap threshold. In practice, however, computational efficiency is constrained not only by runtime but also by hardware memory requirements. We therefore investigate the lowest peak memory usage achievable by each method at each cost gap threshold. Table~\ref{tab:accuracy_memory} reports our findings for fixed $N=M=10^4$ on the same datasets, at the three cost gap thresholds of $1\%$, $5\%$ and $10\%$. Figure~\ref{fig:scalability_memory} reports the peak memory usage as $N,M$ increase from $5\times10^3$ to $5\times10^4$ for Gaussian $d=3,64$, as well as while increasing $d$ up to $1024$.

Although the fused Triton kernels described in Appendix~\ref{app:cuda} keep the Sinkhorn iteration itself memory-efficient, the memory reported here often comes from storing the sparse support $\mathcal S$ of the prior coupling $P^{\mathrm{SOT}}$, whose size grows with the slice count $L$. For example, in Table~\ref{tab:accuracy_memory}, on Gaussian $d{=}64$ at the $1\%$ threshold, SinkSLOT needs $L{=}8192$ slices to bring the cost gap under $1\%$. At $N{=}M{=}10^4$, this many slices covers $55.7\%$ of all $N\times M$ pairs. Since the sparse support must store row and column indices in addition to the corresponding values, its memory footprint can exceed that of a dense $N\times M$ matrix. This explains why SinkSLOT's peak memory of $5050$ MB exceeds SROT's flat, fully dense $2380$ MB in this setting. Nevertheless, in the high-sample-size setting ($N{=}5\times10^4$) shown in Figure~\ref{fig:scalability_memory}, SinkSLOT's memory usage remains manageable even with $4096$ slices, whereas SROT runs out of memory at all slice counts because its dense matrix grows quadratically with $N$.

\appsubsec{Stop-gradient experiments}\label{app:stop_gradient}

We verify the two claims underlying Equation~\eqref{eq:full_grad} numerically: that
term~(I) is the gradient of $\mathrm{SLOT}_\varepsilon$ under the stop-gradient
convention, and that term~(II) is negligible. We also quantify a third,
distinct quantity that is easily conflated with term~(II): the sensitivity of
the Sinkhorn \emph{solve} itself.

\paragraph{Protocol}
All measurements are taken along the blob-to-crescent flow of
Appendix~\ref{app:grad_flow}, with $N=M=1000$, $L=100$ projections,
$\varepsilon$ as stated, $600$ inner Sinkhorn iterations and $50$ gradient
steps at $\eta=0.05$. The trajectory is driven by term~(I) throughout, since
that is the method under study. At each step the reference plan
$P^{\mathrm{SOT}}$ is rebuilt from the current $X$, exactly as the method does,
and is then shared by every estimator evaluated at that step. The projection
directions are drawn from a fixed seed, so across a perturbation only the rank
\emph{orders} they induce can change, never the directions themselves.
Term~(I) is computed in two
independent ways: in closed form as
$2\,\mathrm{diag}(\mathbf{a})(X - T_\varepsilon(X))$ with no automatic
differentiation, and by automatic differentiation through the cost with the
prior coupling held constant. The two agree to a relative $10^{-4}$ at every step, the
discrepancy tracking the marginal violation $\max_i|a_i/r_i-1|$ of the
truncated solve. Reported values are accumulated in \texttt{float32} on GPU,
whose atomic reductions are not order-deterministic, so all numbers below should
be read to two or three significant figures; the finite-difference study is
carried out in \texttt{float64}.

\paragraph{Term~(I) is exact, and term~(II) is a jump rather than a derivative}
We compare the analytic directional derivative along random unit
directions~$v$ against
$[\mathrm{SLOT}_\varepsilon(X+hv) - \mathrm{SLOT}_\varepsilon(X-hv)]/2h$,
with $P^{\mathrm{SOT}}$ either frozen at the base point or rebuilt at each
perturbed point so that rank flips are free to occur. With the reference plan
frozen, the finite difference matches
term~(I) to a relative $1.1\times10^{-9}$ at the first step
(Table~\ref{tab:fd_check}); the growth of the reported error along the flow
reflects the derivative itself decaying to $\sim\!7\times10^{-6}$ against a
fixed finite-difference noise floor, not a degradation of the identity.

With the reference plan rebuilt, the divided difference does not converge as
$h\to0$: at $\varepsilon=0.01$ and the first step it takes the values
$4.7\times10^{-3}$, $1.5\times10^{-2}$, $5.8\times10^{-2}$ and
$2.4\times10^{-1}$ for $h=10^{-3},\dots,10^{-6}$, against an analytic value of
$1.76\times10^{-3}$. This $\mathcal{O}(h^{-1})$ growth is the signature of a
jump discontinuity, not of a missing gradient term, and it is consistent with
$P^{\mathrm{SOT}}$ being piecewise constant: at $h=10^{-7}$, where no rank flip
occurs between the two evaluations, the rebuilt difference returns to the
analytic value. Term~(II) is therefore exactly zero wherever it is defined, and
at the rank-flip boundaries $\mathrm{SLOT}_\varepsilon$ is discontinuous rather
than differentiable.

Those discontinuities are small. Rebuilding $P^{\mathrm{SOT}}$ moves roughly
$0.3\%$ of the support entries and changes the value by a relative
$1.6\times10^{-6}$ at the first step, rising to at most $3.2\times10^{-4}$ at
the last, where the value itself has fallen by two orders of magnitude; the jump
shrinks with $h$ ($4.8\times10^{-4}$ at $h=10^{-3}$ to $3.3\times10^{-5}$ at
$h=10^{-6}$, at step~$50$) and is systematically negative, since the rebuilt
plan is optimal at the perturbed point. Stop-gradient on $P^{\mathrm{SOT}}$ is
thus justified in practice, the justification being piecewise-constancy plus
small jumps rather than differentiability with a vanishing derivative.

\begin{table*}[t]
\centering
\begin{tabular}{rrrrr}
\toprule
step & analytic & FD (frozen) & rel.\ err. & $|\Delta F|/F$ \\
\midrule
 0 & $4.431\times10^{-4}$ & $4.431\times10^{-4}$ & $1.1\times10^{-9}$ & $1.6\times10^{-6}$\\
10 & $-2.906\times10^{-4}$ & $-2.906\times10^{-4}$ & $8.7\times10^{-5}$ & $1.3\times10^{-5}$\\
25 & $-6.904\times10^{-5}$ & $-6.903\times10^{-5}$ & $6.6\times10^{-4}$ & $8.6\times10^{-5}$\\
40 & $-7.862\times10^{-6}$ & $-7.853\times10^{-6}$ & $2.3\times10^{-3}$ & $3.2\times10^{-4}$\\
50 & $-6.683\times10^{-6}$ & $-6.696\times10^{-6}$ & $3.8\times10^{-3}$ & $1.8\times10^{-4}$\\
\bottomrule
\end{tabular}
\caption{Finite-difference check of term~(I) along the flow, $\varepsilon=0.01$,
$h=10^{-4}$, six random directions per point, \texttt{float64}. ``rel.\ err.''
compares the analytic directional derivative to the central difference with
$P^{\mathrm{SOT}}$ frozen; $|\Delta F|/F$ is the relative change in the value
caused by rebuilding $P^{\mathrm{SOT}}$ at the perturbed point, i.e.\ the only
trace term~(II) leaves.}
\label{tab:fd_check}
\end{table*}

\paragraph{The solver term is negligible until convergence}
A separate question is what is lost by treating the \emph{solve} as exact.
Differentiating $\langle C, P\rangle$ through all inner Sinkhorn iterations
yields a gradient that differs from term~(I) by a residual. This residual is
not term~(II), since the reference plan is frozen in both arms, but the
sensitivity of the optimal plan, which the envelope identity for
$\langle C,P\rangle + \varepsilon\,\KL(P\|P^{\mathrm{SOT}})$ absorbs:
unrolling that regularised objective instead recovers term~(I) to a cosine
similarity of $0.999992$.

Figure~\ref{fig:gradient_terms} tracks both norms along the flow. The residual
is non-zero everywhere, but its magnitude is nearly constant, varying by
under $13\%$ from step~$10$ onwards, while term~(I) decays by a factor
of~$102$ as the flow converges. Their ratio therefore grows from $0.029$ at the
first step to $1.12$ at the last, the two crossing at step~$47$. The transport
cost dates the crossover: $\langle P, C\rangle$ falls by nearly three orders of
magnitude and is within a factor of two of its final value by step~$27$, where
the two gradients still agree to a cosine similarity of $0.982$. All of the
disagreement accrues afterwards, over steps in which the objective is no longer
moving. The same three-part picture holds across
$\varepsilon\in\{0.003,0.01,0.03,0.1\}$, with the residual's floor, and hence
the crossover step, set by~$\varepsilon$; at $\varepsilon=0.003$ even $600$
inner iterations do not converge the solve, so that column mixes the effect
with truncation error and is marked accordingly. The exact $W_2$ distance,
computed by dense EMD and so independent of $\varepsilon$ and $L$, plateaus at
the same step as $\langle P, C\rangle$ in every case, confirming that the flow
has genuinely arrived rather than that the regularised objective has saturated.

The analytic gradient is within a few percent of the fully-unrolled one
throughout the phase in which the flow transports mass, and departs from it only
once the objective has settled at its $\mathcal{O}(\varepsilon)$ floor, where
the descent direction no longer matters. Unrolling the solve is \emph{least}
defensible where it is usually assumed to be safest, and its
$\mathcal{O}(k)$ stored activations buy nothing.

\begin{figure*}[t]
\centering
\includegraphics[width=\textwidth]{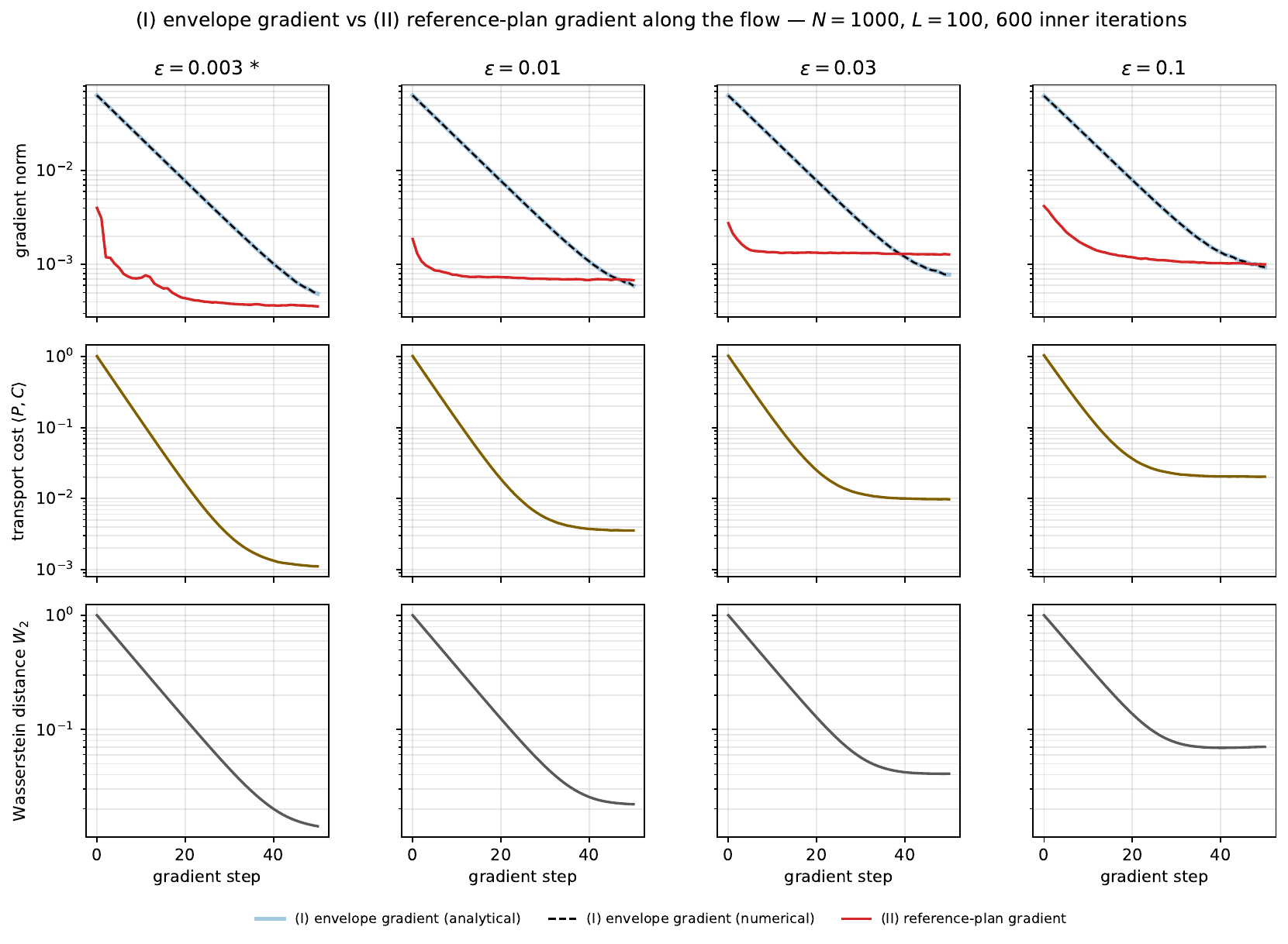}
\caption{Envelope gradient (I) and the solver residual along the
blob-to-crescent flow, one column per $\varepsilon$, $N=M=1000$, $L=100$, $600$
inner iterations. \textbf{Top:} term~(I) computed analytically (closed form) and
numerically (automatic differentiation with the prior coupling held
constant), indistinguishable at this scale, together
with the residual. \textbf{Middle:} the plan's transport cost
$\langle P, C\rangle$. \textbf{Bottom:} the exact $2$-Wasserstein distance $W_2$
to the target. The residual is roughly constant while term~(I) decays by two
orders of magnitude, so it becomes comparable only after the transport cost has
plateaued at its $\varepsilon$-dependent floor. $^{*}$ marks a column whose
solve is not converged at $600$ iterations.}
\label{fig:gradient_terms}
\end{figure*}

\end{document}

%% file: Tables/speedup.tex
\begin{table*}[t]
\centering
\caption{SinkSLOT speedup over competing methods at each transport cost gap threshold. ``---'' indicates that the method did not reach the threshold at convergence (or after 10k Sinkhorn iterations). $N=10^4$. $d=2$ unless otherwise specified. H100-80GB.}
\label{tab:accuracy_speed}
\setlength{\tabcolsep}{2pt}
\small
\resizebox{\textwidth}{!}{%
\begin{tabular}{l ccc ccc ccc ccc ccc}
\toprule
 & \multicolumn{3}{c}{\textbf{Half-moons}}
 & \multicolumn{3}{c}{\textbf{8-Gaussians}}
 & \multicolumn{3}{c}{\textbf{Two-rings}}
 & \multicolumn{3}{c}{\textbf{Gauss $d{=}3$}}
 & \multicolumn{3}{c}{\textbf{Gauss $d{=}64$}} \\
\cmidrule(lr){2-4} \cmidrule(lr){5-7} \cmidrule(lr){8-10} \cmidrule(lr){11-13} \cmidrule(lr){14-16}
Cost gap threshold:
 & ${\leq}1\%$ & ${\leq}5\%$ & ${\leq}10\%$
 & ${\leq}1\%$ & ${\leq}5\%$ & ${\leq}10\%$
 & ${\leq}1\%$ & ${\leq}5\%$ & ${\leq}10\%$
 & ${\leq}1\%$ & ${\leq}5\%$ & ${\leq}10\%$
 & ${\leq}1\%$ & ${\leq}5\%$ & ${\leq}10\%$ \\
\midrule
\multicolumn{16}{l}{\textbf{SinkSLOT speedup ($\mathbf{\times}$):}} \\
FlashSinkhorn~\cite{ye2026flashsinkhorn}  & 13.1 & 8.4 & 8.4 & 17.5 & 19.4 & 16.7 & 36.4 & 14.7 & 8.2 & 4.3 & 9.7 & 9.6 & 2.8 & 6.7 & 11.4 \\
\rowcolor{LightGray}
SROT~\cite{nguyen2026sliced}       & 124.9 & 118.6 & 118.6 & 103.0 & 160.2 & 119.5 & 122.8 & 131.3 & 131.3 & 29.5 & 62.1& 89.7 & 15.2 & 27.1 & 35.1 \\
Spar-Sink~\cite{li2023importance}   & 4.8 & 2.5 & 2.5 & --- & --- & 37.4 & 45.5 & 5.5 & 3.3 & --- & --- & --- & --- & --- & --- \\
\midrule
\rowcolor{lime!8}
\multicolumn{16}{l}{\textbf{Runtime (ms):}} \\
\rowcolor{lime!8}
SinkSLOT (ours) & 16.5 & 7.0 & 7.0 & 145.2 & 18.5 & 11.0 & 7.1 & 5.1 & 5.1 & 1577.9 & 351.0 & 165.7 & 358.9 & 147.9 & 87.2 \\
\bottomrule
\end{tabular}}
\end{table*}

%% file: Tables/convergence.tex
\begin{table*}[t]
\centering
\caption{Iterations run by the SAME configuration selected in the speedup table~\ref{tab:accuracy_speed} (the fastest wall-clock run reaching each transport cost-gap threshold). ``---'' means no feasible run reached the threshold. $N=10^4$. $d=2$ unless otherwise specified. H100-80GB.}
\label{tab:accuracy_iters}
\setlength{\tabcolsep}{3pt}
\small
\resizebox{\textwidth}{!}{%
\begin{tabular}{l ccc ccc ccc ccc ccc}
\toprule
 & \multicolumn{3}{c}{Half-moons}
 & \multicolumn{3}{c}{8-Gaussians}
 & \multicolumn{3}{c}{Two-rings}
 & \multicolumn{3}{c}{Gauss $d{=}3$}
 & \multicolumn{3}{c}{Gauss $d{=}64$} \\
\cmidrule(lr){2-4} \cmidrule(lr){5-7} \cmidrule(lr){8-10} \cmidrule(lr){11-13} \cmidrule(lr){14-16}
Method
 & ${\leq}1\%$ & ${\leq}5\%$ & ${\leq}10\%$
 & ${\leq}1\%$ & ${\leq}5\%$ & ${\leq}10\%$
 & ${\leq}1\%$ & ${\leq}5\%$ & ${\leq}10\%$
 & ${\leq}1\%$ & ${\leq}5\%$ & ${\leq}10\%$
 & ${\leq}1\%$ & ${\leq}5\%$ & ${\leq}10\%$ \\
\midrule
\multicolumn{16}{l}{\textbf{Iterations of the fastest-time (speedup-table) run:}} \\
SinkSLOT (ours) & 420 & 140 & 140 & 3180 & 460 & 250 & 150 & 90 & 90 & 10000 & 3240 & 1430 & 430 & 520 & 990 \\
\rowcolor{LightGray}
FlashSink~\cite{ye2026flashsinkhorn} & 260 & 70 & 70 & 3060 & 430 & 220 & 310 & 90 & 50 & 8220 & 4090 & 1920 & 430 & 430 & 430 \\
SROT~\cite{nguyen2026sliced} & 450 & 140 & 140 & 3210 & 470 & 260 & 150 & 100 & 100 & 7320 & 3290 & 1550 & 1230 & 870 & 700 \\
\rowcolor{LightGray}
Spar-Sink~\cite{li2023importance} & 260 & 70 & 70 & --- & --- & 1010 & 654 & 96 & 50 & --- & --- & --- & --- & --- & --- \\
\bottomrule
\end{tabular}}
\end{table*}

%% file: Tables/memory.tex
\begin{table*}[t]
\centering
\caption{Lowest peak GPU memory (MB) among feasible runs reaching each transport gap threshold. ``---'' means no feasible run reached the threshold. $N=10^4$. $d=2$ unless otherwise specified. H100-80GB.}
\label{tab:accuracy_memory}
\setlength{\tabcolsep}{3pt}
\small
\resizebox{\textwidth}{!}{%
\begin{tabular}{l ccc ccc ccc ccc ccc}
\toprule
 & \multicolumn{3}{c}{Half-moons}
 & \multicolumn{3}{c}{8-Gaussians}
 & \multicolumn{3}{c}{Two-rings}
 & \multicolumn{3}{c}{Gauss $d{=}3$}
 & \multicolumn{3}{c}{Gauss $d{=}64$} \\
\cmidrule(lr){2-4} \cmidrule(lr){5-7} \cmidrule(lr){8-10} \cmidrule(lr){11-13} \cmidrule(lr){14-16}
Method
 & ${\leq}1\%$ & ${\leq}5\%$ & ${\leq}10\%$
 & ${\leq}1\%$ & ${\leq}5\%$ & ${\leq}10\%$
 & ${\leq}1\%$ & ${\leq}5\%$ & ${\leq}10\%$
 & ${\leq}1\%$ & ${\leq}5\%$ & ${\leq}10\%$
 & ${\leq}1\%$ & ${\leq}5\%$ & ${\leq}10\%$ \\
\midrule
\multicolumn{16}{l}{\textbf{Peak memory (MB):}} \\
SinkSLOT (ours) & 820 & 820 & 820 & 896 & 820 & 820 & 820 & 820 & 820 & 2580 & 1799 & 1281 & 5050 & 3007 & 1302 \\
\rowcolor{LightGray}
FlashSink~\cite{ye2026flashsinkhorn} & 1785 & 1785 & 1785 & 1785 & 1785 & 1785 & 1785 & 1785 & 1785 & 1785 & 1785 & 1785 & 1806 & 1806 & 1806 \\
SROT~\cite{nguyen2026sliced} & 2359 & 2359 & 2359 & 2359 & 2359 & 2359 & 2359 & 2359 & 2359 & 2359 & 2359 & 2359 & 2380 & 2380 & 2380 \\
\rowcolor{LightGray}
SparSink~\cite{li2023importance} & 1258 & 1258 & 1258 & --- & --- & 1258 & 1258 & 1258 & 1258 & --- & --- & --- & --- & --- & --- \\
\bottomrule
\end{tabular}}
\end{table*}